\documentclass{article}
\usepackage{graphicx} 

\usepackage{times}
\usepackage{soul}
\usepackage{url}
\usepackage[hidelinks]{hyperref}
\usepackage[utf8]{inputenc}
\usepackage[small]{caption}
\usepackage{graphicx}
\usepackage{amsmath}
\usepackage{amssymb}
\usepackage{mathtools}
\usepackage{amsthm}
\usepackage{booktabs}
\usepackage{algorithm}
\usepackage{algorithmic}
\usepackage{multirow}
\usepackage[switch]{lineno}

\usepackage{subfigure}
\usepackage{color}
\usepackage{authblk}

\newtheorem{proposition}{Proposition}

\begin{document}

\title{Value-Guidance MeanFlow for Offline Multi-Agent Reinforcement Learning}

\author[1]{Teng Pang}
\author[1]{Zhiqiang Dong}
\author[1]{Yan Zhang}
\author[1]{Rongjian Xu}
\author[1 2]{Guoqiang Wu \textsuperscript{*}} 
\author[1]{Yilong Yin \textsuperscript{*}} 

\affil[*]{{\small Corresponding Author}}

\affil[1]{School of Software, Shandong University, Jinan, China}
\affil[2]{Luke EI, Jinan, China

silencept7@gmail.com, zhiqiangdong12@gmail.com, yannzhang9@gmail.com, xurongjian50@gmail.com, guoqiangwu90@gmail.com, ylyin@sdu.edu.cn}


\maketitle

\begin{abstract}
    Offline multi-agent reinforcement learning (MARL) aims to learn the optimal joint policy from pre-collected datasets, requiring a trade-off between maximizing global returns and mitigating distribution shift from offline data. 
    Recent studies use diffusion or flow generative models to capture complex joint policy behaviors among agents; however, they typically rely on multi-step iterative sampling, thereby reducing training and inference efficiency.
    Although further research improves sampling efficiency through methods like distillation, it remains sensitive to the behavior regularization coefficient. 
    To address the above-mentioned issues, we propose Value Guidance Multi-agent MeanFlow Policy (VGM$^2$P), a simple yet effective flow-based policy learning framework that enables efficient action generation with coefficient-insensitive conditional behavior cloning.
    Specifically, VGM$^2$P uses global advantage values to guide agent collaboration, treating optimal policy learning as conditional behavior cloning.
    Additionally, to improve policy expressiveness and inference efficiency in multi-agent scenarios, it leverages classifier-free guidance MeanFlow for both policy training and execution.
    Experiments on tasks with both discrete and continuous action spaces demonstrate that, even when trained solely via conditional behavior cloning, VGM$^2$P efficiently achieves performance comparable to state-of-the-art methods.
\end{abstract}

\section{Introduction}

Multi-agent reinforcement learning (MARL)~\cite{oliehoek2008optimal,rashid2020monotonic,zhang2021multi} is primarily applied to multi-agent system tasks in real-world scenarios, such as multi-player strategy games~\cite{carroll2019utility}, multi-robot control~\cite{peng2021facmac}, and traffic control~\cite{wiering2000multi}. 
The key challenge is how to effectively express powerful policies and facilitate communication among agents during interactions with the environment, thereby maximizing the overall reward of the system.
However, due to the complexity of the real world, real-time interaction with the environment often involves risks and high costs, especially in large-scale tasks. Therefore, offline MARL~\cite{yang2021believe,pmlr-v162-pan22a}, which leverages pre-collected data for multi-agent policy learning, has gradually gained increasing attention.

Similar to single-agent offline RL, offline MARL faces a series of distribution shift challenges. 
First, the limited and insufficient coverage of offline data makes agents more likely to access out-of-distribution (OOD) data during training.
This issue becomes even more challenging as the number of agents grows.
Additionally, the absence of real-time interaction with the environment hampers the proper exploration of the learned policies, thereby exacerbating extrapolation errors.
Beyond these challenges, another key issue is how to effectively mine and utilize the communication between agents from the offline dataset.

To address these challenges, existing research integrates the regularization methods from single-agent offline RL into the Centralized Training with Decentralized Execution (CTDE) framework~\cite{yang2021believe,pmlr-v162-pan22a,shao2023counterfactual,wang2023offline}. This approach ensures communication between agents while limiting OOD data access and mitigating extrapolation errors.
Additionally, recent studies incorporate the impact of agents' balance on policy learning, using sequential policy updates to further restrict OOD data access and extrapolation~\cite{matsunaga2023alberdice,liu2025offline,qiao2025offline}.
While these methods effectively mitigate distribution shifts and communication collaboration issues in multi-agent systems, the commonly used Gaussian policy fails to capture the multi-modal nature of the joint policy, thereby constraining the policy expressiveness of the agents and the scope of their applications. 

With the recent development of generative models, some studies apply models like diffusion~\cite{ho2020denoising} and flow matching~\cite{lipman2023flow} to offline MARL, particularly in policy modeling~\cite{li2024beyond,li2025dof} and trajectory generation~\cite{zhu2024madiff}. 
Although these models are powerful, their complex sampling processes incur high generation costs, and the generated actions cannot be directly used for policy updates. 
Besides, some research studies one-step distillation~\cite{lee2025multi} or one-step generative models~\cite{li2025om2p}, such as MeanFlow~\cite{geng2025mean}, as a behavioral regularization method to efficiently sample and generate optimal actions, but such approaches are highly sensitive to the exploration-exploitation trade-off and heavily dependent on the regularization coefficient. 


To address the aforementioned issues, we propose a simple offline multi-agent policy learning method, \textbf{V}alue \textbf{G}uidance \textbf{M}ulti-agent \textbf{M}eanflow \textbf{P}olicy(\textbf{VGM$^2$P}), which uses the advantage value as guidance and treats training the optimal policy as conditional behavior cloning.
In the training phase, to reduce sensitivity to the exploration–exploitation coefficient,
VGM$^2$P calculates the global advantage value of offline data and integrates it into MeanFlow-based individual policies training with classifier-free guidance (CFG). 
For decentralized execution, to enhance exploration of the learned policies and the efficiency of action generation,
VGM$^2$P generates actions for each agent through one-step sampling based on a preset condition.
Experimentally, we apply VGM$^2$P to general offline MARL benchmarks, and a series of experiments show that VGM$^2$P, using only conditional behavior cloning, performs comparably to existing advanced methods. 

Our contributions are summarized as follows:
\begin{itemize}
    \item We propose VGM$^2$P, a simple yet effective multi-agent training method that trains the optimal joint policy through conditional behavior cloning.
    \item To enhance policy expressiveness and action generation efficiency, we leverage the classifier-free guidance MeanFlow for condition-based behavior cloning. 
    \item To enable agent collaboration, we incorporate the global advantage value as a guidance condition into conditional behavior cloning.
    \item Experimental results demonstrate that, in both discrete and continuous action environments, our method efficiently achieves performance comparable to existing advanced algorithms.
\end{itemize}

\section{Preliminaries}
\subsection{Problem setup}

In this work, we model multi-agent reinforcement learning~(MARL) as a decentralized partially observable Markov decision process (Dec-POMDP) represented by a tuple $\mathcal{M}=(\mathcal{I}, \mathcal{S}, \{ \mathcal{O}^i \}_{i=1}^{N}, \{ \mathcal{A}^i \}_{i=1}^{N}, \Omega, \mathcal{T}, R, \gamma, \rho_0)$. Here, $\mathcal{I}=\{ 1,2,\cdot \cdot \cdot, N\}$ denotes a set of agents; $\mathcal{S}$ denotes the global state space; $\mathcal{O}^i$ and $\mathcal{A}^i$ denote the observation space and action space of the agent $i \in \mathcal{I}$, $\mathbf{A} = \mathcal{A}^1 \times... \times \mathcal{A}^N$ denotes the joint action space and $\mathbf{a}= (a^1, ...,a^N )\in \mathbf{A}$ denotes the joint action, $\mathbf{O}$ and $\mathbf{o}$ similarly represent the corresponding joint observation space and joint observation; $\Omega(s, i): \mathcal{S} \times \mathcal{I} \rightarrow \mathcal{O}^i$ denotes observation function of the agent $i$ that can observe $o^i\in\mathcal{O}^i$ in current state $s\in\mathcal{S}$ and we set $\Omega(s): \mathcal{S} \rightarrow \mathcal{O}^1\times...\times\mathcal{O}^N$ for simplicity; $\mathcal{T}(s'|s,\mathbf{a}): \mathcal{S} \times \mathbf{A} \times \mathcal{S} \rightarrow [0, 1]$ denotes the state transition function; $R(s,\mathbf{a}): \mathcal{S} \times \mathbf{A}  \rightarrow \mathbb{R}$ denotes the global reward model and $r_t=R(s_t,\mathbf{a}_t)$ denotes the global reward at time $t$; $\gamma$ is the discounted factor and $\rho_0$ is the initial state distribution. 
In Dec-POMDP, each agent $i$ can only observe $o^i_t$ at each transition time $t$ and execute the action $a^i_t$ according to its own policy $\pi^i(a^i_t|o^i_t):\mathcal{O}^i \times \mathcal{A}^i \rightarrow [0, 1]$. The goal of MARL is to learn the joint policy $\pi^{\mathrm{tot}}=(\pi^1, ..., \pi^N)$ that maximizes the discounted cumulative reward $J_{\pi^{\mathrm{tot}}}=\mathbb{E}_{s_0\sim\rho_0,\pi^{\mathrm{tot}}(\cdot|\Omega(s_t)),\mathcal{T}} [\sum_{t=0}^{\infty}\gamma^{t}r_t]$. 
For the joint policy $\pi^{\mathrm{tot}}$, we have a global Q-value function $Q^{\mathrm{tot}}_{\pi^{\mathrm{tot}}}(\mathbf{o},\mathbf{a})=\mathbb{E}_{\pi^{\mathrm{tot}}(\cdot|\Omega(s_t)), \mathcal{T}}[\sum_{t=0}^{\infty} \gamma^t r_t|\mathbf{o}_0=\mathbf{o},\mathbf{a}_0=\mathbf{a}]$
and its corresponding value function $V^{\mathrm{tot}}_{\pi^{\mathrm{tot}}}(\mathbf{o})= \mathbb{E}_{\mathbf{a} \sim \pi^{tot}} [Q^{\mathrm{tot}}_{\pi^{\mathrm{tot}}}(\mathbf{o},\mathbf{a})]$.
In offline scenarios, we have a static dataset $\mathcal{D}_{\mathrm{off}}=\{ \mathcal{D}^i_{\beta} \}^N_{i=1}$ collected by $N$ agents following behavior joint policy $\pi_{\beta}^{\mathrm{tot}}=(\pi^1_{\beta}, ..., \pi^N_{\beta})$.
Each agent $i$ provides a sub-dataset $\mathcal{D}^i_{\beta}=\{(o_m^{i}, a_m^{i},{o'}_{m}^{i},r_m)\}_{m=1}^{M}$ consisting of $M$ transition tuples.
For the single-agent case, we drop the agent identifier and denote $V_{\pi}(o)$, $Q_{\pi}(o,a)$, $\mathcal{D}_{\beta}$, and $\pi_{\beta}$ for simplicity.

\subsection{Centralized Training with Decentralized Execution}
Centralized Training with Decentralized Execution (CTDE)~\cite{hernandez2019survey} is a widely adopted training paradigm in MARL, where agents are trained jointly and execute independently at inference time. Under CTDE, value decomposition~\cite{sunehag2018value,son2019qtran}, as a commonly used training method, improves scalability by decomposing the joint observation-action space. This method typically relies on the Individual-Global-Max (IGM) principle~\cite{rashid2020monotonic}, which
requires that combining the individually optimal actions implied by each agent’s Q-value function $Q^i_{\pi^i}(o^i,a^i)$ yields the optimal joint action:
\begin{align}
    &\arg\max_{\mathbf{a}}Q^{\mathrm{tot}}_{\pi^{\mathrm{tot}}}(\mathbf{o},\mathbf{a}) \notag \\
    &= ( \arg\max_{a^1}Q^1_{\pi^{1}}(o^1,a^1),...,\arg\max_{a^N}Q^N_{\pi^{N}}(o^N,a^N) ).
\end{align}
The IGM principle guarantees consistency between local optima and the global optimum.

\subsection{Flow Matching and MeanFlow}
Flow Matching~\cite{lipman2023flow} is a generative model that learns a velocity field to match the flow between a prior distribution and a target distribution. Formally, given data $x \sim p_{\mathrm{target}}$ and prior $\epsilon \sim p_{\mathrm{prior}}$ (e.g., $\epsilon \sim \mathcal{N}(0, \mathbf{I})$), we consider a linear schedule flow path $x_t=(1-t)x + t\epsilon$ at time $t \in [0,1]$, which can lead to the sample-conditional velocity $v_c=\epsilon-x$ by computing the time-derivative. 

In Flow Matching, the parameterized velocity network $v_{\theta}$ is optimized by minimizing the following loss function:
\begin{align}
    \mathcal{L}_{\mathrm{FM}}(\theta)=\mathbb{E}_{x,t,\epsilon}|| v_{\theta}(x_t,t)-(\epsilon - x) ||^2 ,
\end{align}
where $t$ is sampled from the uniform distribution (i.e., $ t\sim \mathrm{Unif}([0,1])$) and $\epsilon$ is sampled from Gaussian distribution (i.e., $\epsilon \sim \mathcal{N}(0, \mathbf{I})$).
Since an intermediate sample $x_t$ can be formed as different $(x,\epsilon)$ pairs, Flow Matching essentially learns a marginal velocity field $v(x_t,t)\triangleq \mathbb{E}_{p(x_t|x,\epsilon)}[v_c|x_t]$ over all possibilities $p(x_t|x,\epsilon)$.
The generative process in Flow Matching is described by the ordinary differential equation (ODE) $\frac{d}{dt}x_t=v_{\theta}(x_t,t)$ for $x_t$. This ODE starts from $x_1=\epsilon$ to $x_0=x$.

Unlike Flow Matching that models instantaneous velocity $v(x_t,t)$, MeanFlow~\cite{geng2025mean} defines an average velocity between two time points $t$ and $r$:
\begin{align}\label{define_meanflow}
    u_{\theta}(x_t,r,t) \triangleq \frac{1}{t-r}\int^{t}_{r}v(x_{\tau}, \tau)d\tau.
\end{align}
To learn the average velocity, Meanflow models it with a parameterized network $u_{\theta}$ and trains with the following loss:
\begin{align}
    \mathcal{L}_{\mathrm{MF}}(\theta)=\mathbb{E}_{x,t,r,\epsilon}|| u_{\theta}(x_t,r,t)-\mathrm{sg}(u_{\mathrm{tgt}}) ||^2,
\end{align}
where $ (t,r)\sim \mathrm{Unif}([0,1])$, $\mathrm{sg}$ denotes a stop-gradient operation and $u_{\mathrm{tgt}}=v(x_t,t)-(t-r)\frac{d}{dt}u(x_t,r,t)$ is the target
velocity. To compute the $\frac{d}{dt}u(x_t,r,t)$, MeanFlow further extends this partial derivative and finally implements the calculation using the Jacobian-vector product (JVP): 
\begin{align}
    \frac{d}{dt}u(x_t,r,t)=v(x_t,t)\partial_{x}u(x_t,r,t) + \partial_{t}u(x_t,r,t).
\end{align}
After training, MeanFlow can achieve one-step sampling, $x_0=x_1-u_{\theta}(x_1,0,1)$, by simply setting $(r,t)=(0,1)$.

\subsection{Behavior Regularization in Offline RL}

In single-agent offline RL, policy training is typically achieved by the actor-critic framework under behavior regularization, resulting in a form of constrained policy optimization with behavior policy $\pi_{\beta}$ or offline dataset $\mathcal{D}_{\beta}$:
\begin{align}\label{optim:offline_value}
    \mathcal{L}_{Q_{\pi}}=\mathbb{E}_{(o,a,o',r)\sim \mathcal{D}_{\beta}, \atop a'\sim \pi(\cdot|o')}[Q_{\pi}(o,a)-(r+\gamma Q_{\pi}(o',a'))]^2,\\
    \label{optim:offline_policy}
    \mathcal{L}_{\pi}=-\mathbb{E}_{(o,a)\sim \mathcal{D}_{\beta},\atop a_{\pi}\sim \pi(\cdot|o)} \big [ Q_{\pi}(o,a_{\pi})-\lambda D_{f}(\pi(o|s)||\pi_{\beta}(o|s)) \big],
\end{align}
where $D_{f}$ is used to measure the divergence between the policy $\pi$ and the behavior policy $\pi_{\beta}$.

To prevent excessive access to OOD actions during training, some works focus on weighted behavior cloning, such as AWAC~\cite{peng2019advantage,nair2020awac}, which derives the representation of the optimal policy $\pi^*$ of Eq.~\eqref{optim:offline_policy} when $D_{f}$ is the KL divergence:
\begin{align}\label{optim:offlinePI}
    \pi^*(a|s)=\frac{\exp(\frac{1}{\lambda}Q_{\pi}(s,a))}{\int_{a'} \pi_{\beta}(a'|s) \exp(\frac{1}{\lambda}Q_{\pi}(s,a'))\mathrm{d}a'} \pi_{\beta}(a|s).
\end{align}

\section{Methodology}

In this section, we present our method VGM$^2$P, a simple yet effective way to represent the optimal joint policy through conditional behavior cloning with MeanFlow for offline MARL. 
Based on the IGM principle, the optimal joint action can be derived from the optimal actions of individual agents.
Therefore, to get the optimal joint action, we can first obtain the optimal policy for each agent and then use the IGM principle to derive the joint optimal policy.
To achieve the above objective, VGM$^2$P consists of the following three aspects: 1) deriving the optimal policy for each agent through behavior policy conditioning on the advantage value; 2) modeling these policies with MeanFlow; and 3) obtaining the joint optimal policy based on the IGM principle.

\subsection{Value Guidance Behavior Policy}\label{section:bc_vg}
In single-agent offline RL, policy improvement depends on the behavior policy~\cite{nair2020awac,wangdiffusion,pmlr-v267-park25f}. When the policy is represented as a distribution, the optimal policy and the behavior policy are positively correlated, as shown in Eq.\eqref{optim:offlinePI}. According to Bayes' theorem, there is a similar correlation between the conditional distribution and its corresponding prior distribution. Based on this insight, we can use a conditional behavior policy to approximate the optimal policy:
\begin{proposition}[Value-Guidance Behavior Policy]\label{proposition:conditional behavior policy}
Given a behavior policy $\pi_{\beta}(a|o)$ and the optimal policy $\pi^*(a|o)$ derived from Eq.\eqref{optim:offlinePI}, for any variable $c\in C$ and its related distribution $p(c|o,a)$, when there exists $c^* \in C$ satisfying $p(c=c^*|o,a)\propto \exp(\frac{1}{\lambda}Q_{\pi}(o,a))$, then we have the conditional behavior policy $\pi_{\beta}(a|o,c=c^*)=\pi^*(a|o)$.
\end{proposition}
The proof is provided in Appendix~\ref{appendix:proof_conditional behavior policy}. Proposition~\ref{proposition:conditional behavior policy} implies that, when we have a condition $c^*$ positively correlated with the value $\exp(\frac{1}{\lambda}Q_{\pi}(o,a))$ to control behavior policy sampling, we can achieve the same distribution as the optimal policy $\pi^{*}(a|o)$ derived from Eq.\eqref{optim:offlinePI}. 

To achieve $p(c=c^*|o,a)\propto \exp(\frac{1}{\lambda}Q_{\pi}(o,a))$,  we can define the advantage value function $A_{\pi}(o,a)=Q_{\pi}(o,a)-V_{\pi}(o)$ and simply set $c=1$ if $A_{\pi}(o,a) \ge 0$ else $c=0$, which is similar to~\cite{frans2025diffusion} in single-agent scenario. Then, we can define
\begin{align}
    p(c=1|o,a) :=  \frac{\exp(\frac{1}{\lambda}Q_{\pi}(o,a))}{\exp(\frac{1}{\lambda}Q_{\pi}(o,a)) + \exp(\frac{1}{\lambda}V_{\pi}(o))}.
\end{align}
Intuitively, during training, when $(o, a)$ has a non-negative advantage value, we set $c=1$ to indicate that the action $a$ comes from the optimal policy. Then, for the execution, we can fix $c=1$ to sample from the optimal policy.

In multi-agent settings, we can replace the local value function $Q_{\pi^i}^i$ with a global value function $Q_{\pi^{\mathrm{tot}}}^\mathrm{tot}$ and use the global advantage value $A_{\pi^{\mathrm{tot}}}^\mathrm{tot} (\mathbf{o}, \mathbf{a}) = Q_{\pi^{\mathrm{tot}}}^\mathrm{tot}(\mathbf{o}, \mathbf{a})-V_{\pi^{\mathrm{tot}}}^\mathrm{tot}(\mathbf{o})$ as the guidance condition, enabling cooperative policy execution, which we will discuss in detail in Section~\ref{section:mamf_vg}.

\subsection{Value Guidance MeanFlow Policy}\label{section:mf_vg}

To enhance the expressive ability of the policy, we present using MeanFlow to model it. As a specific implementation of Continuous Normalizing Flows (CNFs), MeanFlow has been widely adopted in image generation due to its ability to achieve both high efficiency and high-quality sample generation. 
Specifically, for a single-agent and its observation-action pairs $(o,a)\sim\mathcal{D_{\beta}}$, we construct the action $a_k=(1-k)a+k\epsilon$ at the timestep $k\in[0,1]$ of the flow process along with its sample-conditional velocity $v_{c}(a_k,k|o)=\epsilon-a$ and the parameterized average velocity $u_{\theta}(a_k,r,k|o)=\frac{1}{k-r}\int^{k}_{r}v_{c}(a_k,k|o)$. During training, we formulate a MeanFlow-based behavior cloning loss as follows:
\begin{align}
    \mathcal{L}_{\mathrm{BC-MF}}(\theta)=\mathbb{E}_{(o,a)\sim\mathcal{D_{\beta}}, k,r,\epsilon}|| u_{\theta}(a_k,r,k|o)-\mathrm{sg}(u_{\mathrm{tgt}}) ||^2,
\end{align}
where $(k,r)\sim \mathrm{Unif}([0,1])$, $\epsilon \sim \mathcal{N}(0, \mathbf{I})$ and  $u_{\mathrm{tgt}}=v_{c}(a_k,k|o)-(k-r)\frac{d}{dk}u_{\theta}(a_k,r,k|o)$ is the target
velocity.

Based on the Proposition~\ref{proposition:conditional behavior policy}, the optimal policy derived from Eq.\eqref{optim:offlinePI} can be approximated by a behavior policy augmented with value-guidance conditioning. 
Based on it, we propose the Value-Guided MeanFlow Policy(VGMP), which is trained via a parameterized conditional average velocity $u^c_{\theta}(a_k,r,k|o,c)$ and optimized through behavior cloning under the Classifier-free Guidance (CFG) MeanFlow:
\begin{align}\label{optimal:VGMP}
    \mathcal{L}_{\mathrm{VGMP}}(\theta)=\mathbb{E}_{(o,a)\sim\mathcal{D_{\beta}}, k,r,\epsilon}|| u^c_{\theta}(a_k,r,k|o,c)-\mathrm{sg}(u_{\mathrm{tgt}}^{\mathrm{cfg}}) ||^2,
\end{align}
where $c$ is the value-guidance condition, $u_{\mathrm{tgt}}^{\mathrm{cfg}} = v_{\mathrm{cfg}}(a_k,k|o,c) - (k-r) \frac{d}{dk} u^c_{\theta}(a_k,r,k|o,c)$ is the target velocity and $v_{\mathrm{cfg}}(a_k,k|o,c)=\omega v_c(a_k,k|o,c)+(1-\omega)u^c_{\theta}(a_k,k,k|o)$ is the ground-truth field that a weighted combination of the class-conditional field $v_c(a_k,k|o,c)$ and the class-unconditional field $u^c_{\theta}(a_k,k,k|o)$ with a guidance weight $\omega$ under CFG. Following~\cite{geng2025mean}, we replace the class-conditional field $v_c(a_k,k|o,c)$ with the sample-conditional velocity $v_c(a_k,k|o) = \epsilon - a$ and set the class-unconditional field $u^c_{\theta}(a_k,k,k|o)=u^c_{\theta}(a_k,k,k|o,c=1)$ to allow exploration of the offline dataset's behavior even when using the optimal policy. 

For the execution, we sample $a_1 \sim\mathcal{N}(0, \mathbf{I})$ and set $c=1$, generating each action with the conditional average velocity:
\begin{align}
    a_r=a_k - (k-r) u_{\theta}^c(a_k,r,k|o,1).
\end{align}
To improve generation efficiency, we can directly use one-step sampling, i.e., $a=a_0=a_1 - u_{\theta}^c(a_1,0,1|o,1)$.

\subsection{Value Guidance Multi-agent MeanFlow Policy}\label{section:mamf_vg}

In MARL, policy learning seeks to maximize the global value of joint actions, which requires explicitly obtaining and leveraging the global value during value guidance. 
However, achieving these goals suffers from two key limitations: first, the computational cost of the joint observation-action space typically grows exponentially with the number of agents; second, as a scalar, the global value cannot provide agent-specific guidance to all agents simultaneously.

Under the IGM principle, the joint action that maximizes the global Q-value can be decomposed into actions that individually maximize each agent’s local Q-value. Therefore, to reduce computational cost and satisfy the IGM principle, we replace the global Q-value with individual agents’ Q-values. Specifically, we approximate the global Q-value function $Q^{\mathrm{tot}}_{\pi^{\mathrm{tot}}}(\mathbf{o},\mathbf{a})$ by summing parameterized individual agents’ value functions $Q^{i}_{\phi_i}(o^i,a^i)$, like VDN~\cite{sunehag2018value}, i.e., $Q^{\mathrm{tot}}_{\pi^{\mathrm{tot}}}(\mathbf{o},\mathbf{a})=\sum_{i=1}^{N} Q^i_{\phi_i} (o^i,a^i)$, and then train them using global Temporal-Difference (TD) error:
\begin{align}\label{optim:offlineMARL_value}
    \mathcal{L}(\{\phi_i\}_{i = 1}^N) &= \mathbb{E}_{(\mathbf{o}, \mathbf{a},\mathbf{o}',r) \sim \mathcal{D}_{\mathrm{off}}, \atop \mathbf{a}'^i\sim \pi^{\mathrm{tot}}(\cdot|\mathbf{o}')}  
    \Big [\sum_{i=1}^{N} \bigl( Q^i_{\phi_i} (o^i,a^i) \notag \\
    &- ( r + \gamma Q^i_{\bar{\phi}_i}(o'^i,a'^i) ) \bigr ) \Big]^2, 
\end{align}
where $Q^i_{\bar{\phi}_i}$ denotes a slowly updated target Q-value network for stabilizing training. 

Since we approximate the global Q-value with individual agents’ Q-values, we can use these Q-values to guide the training of conditional behavior cloning. 

\begin{proposition}\label{proposition:policy decomposition} 
Assuming that the behavior joint policy $\pi_{\beta}^{\mathrm{tot}}(\mathbf{a}|\mathbf{o})$ and the global Q-value $Q^{\mathrm{tot}}_{\pi^{\mathrm{tot}}}(\mathbf{o},\mathbf{a})$ are decomposable, i.e., $\pi_{\beta}^{\mathrm{tot}}(\mathbf{a}|\mathbf{o}) = \prod_{i=1}^{N} \pi^{i}_{\beta}(a^i|o^i)$ and $Q^{\mathrm{tot}}_{\pi^{\mathrm{tot}}}(\mathbf{o},\mathbf{a})=\sum_{i=1}^{N} Q^i_{\phi_i} (o^i,a^i)$. 
For the optimal joint policy $\pi^{\mathrm{tot},*}(\mathbf{a}|\mathbf{o})$, when the distribution $p^i(c^i|o^i,a^i)$ satisfies $p^i(c^i=c^{i,*}|o^i,a^i)\propto \exp(\frac{1}{\lambda}Q^i_{\phi_i}(o^i,a^i))$ for each agent $i$, then we have $\prod_{i=1}^{N} \pi^{i}_{\beta}(a^i|o^i,c^i=c^{i,*})=\pi^{\mathrm{tot},*}(\mathbf{a}|\mathbf{o})$.
\end{proposition}

The proof is provided in Appendix~\ref{appendix:proof_policy decomposition}.
Proposition~\ref{proposition:policy decomposition} indicates that under the value decomposition, the joint behavior policy with value guidance condition is the joint optimal policy $\pi^{\mathrm{tot},*}(\mathbf{a}|\mathbf{o})$. 
Therefore, by considering the optimality of individual agents under the IGM principle, we can achieve the global optimum.

In terms of implementation, we use the advantage value $A^i (o^i, a^i) = Q^i_{\phi_i} (o^i,a^i)-Q^i_{\phi_i} (o^i,\hat{a}^i)$ where $\hat{a}^i \sim \pi^i(\cdot|o^i)$ to set the condition $c^i$ and fix the condition $c^i=1$ during execution for each agent.
In addition, to further improve agent communication and collaboration, we share the parameters of both the policy network and the value function network across all agents.
Combining the use of MeanFlow, we name the above approach Value Guidance Multi-agent MeanFlow Policy (VGM$^2$P), and present the complete training and execution processes in Algorithm~\ref{alg:training} and~\ref{alg:execution}, respectively.

\begin{algorithm}[tb]
    \caption{Centralized Training of VGM$^2$P}
    \label{alg:training}
    \textbf{Input}: Offline MARL dataset $\mathcal{D}_{\mathrm{off}}$; individual conditional average velocity model $\{ u_{\theta_i}^c \}_{i=1}^N$, individual Q-value model $\{ Q_{\phi_i}^i \}_{i=1}^N$; guidance weight $\omega$\\
    \textbf{Output}: $\{ u_{\theta_i}^c \}_{i=1}^N$
    \begin{algorithmic}[1] 
    \STATE Initialize $\{ u_{\theta_i}^c \}_{i=1}^N$ and $\{ Q_{\phi_i}^i \}_{i=1}^N$
    \WHILE{not converged}
        \STATE Sample tuple $\{(o^i,a^i,o'^i,r)\}^{N}_{i=1} \sim \mathcal{D}_{\mathrm{off}}$
        \STATE // Train individual Q-value model $\{ Q_{\phi_i}^i \}_{i=1}^N$
        \STATE Sample $a'^i_1 \sim\mathcal{N}(0, \mathbf{I})$, set $a'^i=a'^i_1-u_{\theta_i}^c(a'^i_1,0,1|o'^i,1)$
        \STATE Train Q-value model $\{ Q_{\phi_i}^i \}_{i=1}^N$ with Eq.~\eqref{optim:offlineMARL_value}
        \STATE // Train individual conditional average velocity model $\{ u_{\theta_i}^c \}_{i=1}^N$
        \STATE Sample $\hat{a}_1^i \sim\mathcal{N}(0, \mathbf{I})$ and generate current action $\hat{a}^i=\hat{a}^i_1-u_{\theta_i}^c(\hat{a}^i_1,0,1|o^i,1)$
        \STATE Compute advantage value $A^i(o^i,a^i)=Q^i_{\phi_i}(o^i,a^i)-Q^i_{\phi_i}(o^i,\hat{a}^i)$ and set the condition $c^i=1$ if $A^i(o^i,a^i) \ge V(o^i)$ else $c^i=0$
        \STATE Sample $a_1^i \sim\mathcal{N}(0, \mathbf{I}),(k,r)\sim \mathrm{Unif}([0,1])$
        \STATE Train conditional average velocity model $\{ u_{\theta_i}^c \}_{i=1}^N$ with Eq.~\eqref{optimal:VGMP}
        \ENDWHILE
    \end{algorithmic}
\end{algorithm}

\begin{algorithm}[tb]
    \caption{Decentralized Execution of VGM$^2$P}
    \label{alg:execution}
    \textbf{Input}: local observation $\{o^i\}_{i=1}^N$, conditional average velocity model $\{ u_{\theta_i}^c \}_{i=1}^N$\\
    \textbf{Output}: $\{a^i\}_{i=1}^N$
    \begin{algorithmic}[1] 
    \STATE Sample $\{a^i_1\}_{i=1}^{N} \sim\mathcal{N}(0, \mathbf{I})$
    \STATE Compute $a^i=a^i_1-u_{\theta_i}^c(a^i_1,0,1|o^i,1)$ for each agent $i$
    \end{algorithmic}
\end{algorithm}

\begin{table*}[t]
  \small
  \center
  \scalebox{0.55}{
  \begin{tabular}{cl|rrrrrr|rrrr}
  \toprule
  \multirow{2}{*}{}&   \multirow{2}{*}{Dataset}&  \multicolumn{6}{c|}{Extension of offline SARL}&  \multicolumn{4}{c}{Offline MARL}\\
  & & BC(Gaussian)& BC(Diffusion)& BC(FM)& BC(MF)& MA-BCQ& MA-CQL& MADiff& DoF& MAC-Flow& VGM$^2$P \\
  \midrule
  \multirow{13}{*}{\rotatebox{90}{SMACv1}}
  & 3m-Good & 16.0$\pm$1.0& 19.5$\pm$0.5& \textbf{20.0$\pm$0.0}& \underline{19.8$\pm$0.4}& 3.7$\pm$1.1& 19.1$\pm$0.1& 19.3$\pm$0.5& \underline{19.8}$\pm$0.2& \underline{19.8$\pm$0.2}& 19.5$\pm$0.7\\ 
  & 3m-Medium & 8.2$\pm$0.8& 13.3$\pm$0.7& 14.7$\pm$1.5& 15.0$\pm$2.8& 4.0$\pm$1.0& 13.7$\pm$0.3& 16.4$\pm$2.6& \textbf{18.6$\pm$1.2}& \underline{18.0$\pm$3.2}& 16.9$\pm$1.1\\
  & 3m-Poor & 4.4$\pm$0.1& 4.2$\pm$0.2& 4.5$\pm$0.1& 4.2$\pm$0.3& 3.4$\pm$1.0& 4.2$\pm$0.1&
  10.3$\pm$6.1& \underline{10.9$\pm$1.1}& 10.6$\pm$2.2& \textbf{14.9$\pm$1.5}\\
  \cmidrule{2-12}
  & 8m-Good & 16.7$\pm$0.4& 19.4$\pm$0.5& 19.5$\pm$0.2& 19.5$\pm$0.6& 4.8$\pm$0.6& 18.9$\pm$0.9& 18.9$\pm$1.1& \underline{19.6$\pm$0.3}& \textbf{19.7$\pm$0.3}& \textbf{19.7$\pm$0.4}\\
  & 8m-Medium & 10.7$\pm$0.5& 18.6$\pm$0.6& 18.2$\pm$0.8& \underline{18.7$\pm$0.8}& 5.6$\pm$0.6& 15.5$\pm$1.5& 16.8$\pm$1.6& 18.6$\pm$0.8& \textbf{19.4$\pm$0.6}& 18.2$\pm$1.6\\
  & 8m-Poor & 5.3$\pm$0.1& 4.8$\pm$0.2& 4.9$\pm$0.1& 4.8$\pm$0.1& 3.6$\pm$0.8& 7.5$\pm$1.0&
  9.8$\pm$0.9& \textbf{12.0$\pm$1.2}& \underline{11.5$\pm$0.8}& 4.9$\pm$0.1\\
  \cmidrule{2-12}
  & 2s3z-Good & 18.2$\pm$0.4& 18.0$\pm$1.0& \underline{19.5$\pm$0.1}& 19.1$\pm$0.9& 7.7$\pm$0.9& 17.4$\pm$0.3& 15.9$\pm$1.2& 18.5$\pm$0.8& \underline{19.5$\pm$0.5}& \textbf{19.9$\pm$0.1}\\
  & 2s3z-Medium & 12.3$\pm$0.7& 13.4$\pm$1.4& 15.1$\pm$2.0& 14.3$\pm$1.8& 7.6$\pm$0.7& 15.6$\pm$0.4& 15.6$\pm$0.3& \textbf{18.1$\pm$0.9}& \underline{17.6$\pm$0.6}& 16.5$\pm$0.6\\
  & 2s3z-Poor & 6.7$\pm$0.3& 6.2$\pm$1.2& 6.9$\pm$0.8& 7.0$\pm$1.0& 6.6$\pm$0.2& 8.4$\pm$0.8&
  \underline{8.5$\pm$1.3}& \textbf{10.0$\pm$1.1}& \underline{8.5$\pm$0.6}& 7.9$\pm$0.7\\
  \cmidrule{2-12}
  & 5m\_vs\_6m-Good & 15.8$\pm$3.6& 16.8$\pm$2.3& 14.7$\pm$2.1& 14.9$\pm$3.2& 2.4$\pm$0.4& 16.2$\pm$1.6& 16.5$\pm$2.8& \underline{17.7$\pm$1.1}& \textbf{18.6$\pm$3.5}& 17.6$\pm$1.3\\
  & 5m\_vs\_6m-Medium & 12.4$\pm$0.9& 12.5$\pm$2.1& 12.8$\pm$0.8& 13.5$\pm$2.2& 3.8$\pm$0.5& 15.1$\pm$2.9& 15.2$\pm$2.6& \underline{16.2$\pm$0.9}& 15.6$\pm$1.3& \textbf{17.0$\pm$0.9}\\
  & 5m\_vs\_6m-Poor & 7.5$\pm$0.2& 8.0$\pm$1.0& 7.7$\pm$0.8& 8.4$\pm$1.1& 3.3$\pm$0.5& 10.5$\pm$3.1& 8.9$\pm$1.3& \textbf{10.8$\pm$0.3}& 9.8$\pm$2.1& \underline{10.7$\pm$1.1}\\
  \cmidrule{2-12}
  & Average&    11.2& 12.9& 13.2& 13.2&  4.7& 13.5& 14.3& \textbf{15.9}& \underline{15.7}&  15.3\\
  \midrule
  \multirow{5}{*}{\rotatebox{90}{SMACv2}} 
  & terran\_5\_vs\_5-Replay& 7.3$\pm$1.0&  9.3$\pm$0.9&  8.3$\pm$1.9&  9.3$\pm$2.0&  13.8$\pm$4.4& 11.8$\pm$0.9& 13.3$\pm$1.8& \underline{15.4}$\pm$1.3& \textbf{16.6$\pm$4.3}& 12.2$\pm$1.8\\
  & zerg\_5\_vs\_5-Replay& 6.8$\pm$0.6&  8.1$\pm$1.7&  4.6$\pm$0.5&  6.2$\pm$0.4& \underline{10.3$\pm$1.2}& \underline{10.3$\pm$3.4}& 10.2$\pm$1.1& \textbf{12.0$\pm$1.1}& 9.8$\pm$1.5& 9.6$\pm$4.1\\
  & terran\_10\_vs\_10-Replay& 7.4$\pm$0.5&  5.5$\pm$1.5&  5.8$\pm$1.7& 5.6$\pm$0.6& 12.7$\pm$2.0&  11.8$\pm$2.0& \underline{13.8$\pm$1.3}&  \textbf{14.6$\pm$1.1}& 13.0$\pm$4.7& 7.7$\pm$0.8\\
  \cmidrule{2-12}
  & Average&  7.2& 7.6& 6.2& 7.0& 12.3& 11.3& 12.4& \textbf{14.0}& \underline{13.1}& 9.8\\
  \bottomrule
  \end{tabular}
  }
  \caption{Comparative performance of VGM$^2$P with discrete actions environment. For the SMACv1 environment, we select 4 tasks, each with 3 datasets of varying quality. For the SMACv2 environment, we select three tasks with only 1 dataset. 
  To distinguish different Behavior Cloning methods and simplify notation, we use \textit{FM} and \textit{MF} to represent Flow Matching and MeanFlow, respectively. We report the average performances and standard deviations of each algorithm across 6 seeds, with the best result in \textbf{bold} and the second-best result \underline{underlined}.}
  \label{result:discrete_action}
  \vskip -0.1in
\end{table*}

\begin{table*}[t]
  \small
  \center
  \scalebox{0.55}{
  \begin{tabular}{cl|rr|rrrrrr}
  \toprule
  \multirow{2}{*}{}&   \multirow{2}{*}{Dataset}&  \multicolumn{2}{c|}{Extension of offline SARL}&  \multicolumn{6}{c}{Offline MARL}\\
  & &  MA-TD3BC& MA-CQL& MA-ICQ& OMAR& OMIGA& MADiff& MAC-Flow& VGM$^2$P \\
  \midrule
  \multirow{13}{*}{\rotatebox{90}{MA-MuJoCo}} 
  & 6Halfcheetah-Expert&   4401.6$\pm$169.1& 4589.5$\pm$98.5& 
  2955.9$\pm$459.2& -206.7$\pm$161.1& 3383.6$\pm$552.7& \underline{4711.4$\pm$213.6}& 4650.0$\pm$271.6& \textbf{4897.5$\pm$114.5}\\
  & 6Halfcheetah-Medium& 2620.8$\pm$69.9& 3189.4$\pm$306.9&  
  2549.3$\pm$96.3& -265.7$\pm$147.0& 3608.1$\pm$237.4& 2650.0$\pm$365.4& \textbf{4358.5$\pm$369.2}& \underline{3684.8$\pm$130.4}\\
  & 6Halfcheetah-MR&      \underline{3528.9$\pm$120.9}& 3500.7$\pm$293.9&
  1922.4$\pm$612.9& -235.4$\pm$154.9& 2504.7$\pm$83.5& 2830.5$\pm$292.8& 3030.2$\pm$436.8&
  \textbf{4068.5$\pm$113.5}\\
  & 6Halfcheetah-ME&      3518.1$\pm$381.0& 4738.2$\pm$181.1&
  2834.0$\pm$420.3& -253.8$\pm$63.9& 2948.5$\pm$518.9& 4410.9$\pm$836.8& \underline{5139.9$\pm$84.1}& \textbf{5159.2$\pm$156.3}\\
  \cmidrule{2-10}
  & 3Hopper-Expert&      3309.9$\pm$4.5& \underline{3359.1$\pm$513.8}&
  754.7$\pm$806.3& 2.4$\pm$1.5& 859.6$\pm$709.5& 2853.3$\pm$593.8& \textbf{3592.1$\pm$8.9}&
  2473.5$\pm$876.6\\
  & 3Hopper-Medium&       870.4$\pm$156.7& 901.3$\pm$199.9&
  501.8$\pm$14.0& 21.3$\pm$24.9& 1189.3$\pm$544.3& \underline{1436.8$\pm$449.5}& 1023.5$\pm$253.0& \textbf{2008.6$\pm$1389.4}\\
  & 3Hopper-MR&           269.7$\pm$41.8& 31.4$\pm$15.2& 
  195.4$\pm$103.6& 3.3$\pm$3.2& 774.2$\pm$494.3& 936.1$\pm$574.0& \underline{1166.3$\pm$451.9}& \textbf{1426.6$\pm$665.5}\\
  & 3Hopper-ME&         2904.3$\pm$477.4& 2751.8$\pm$123.3& 
  355.4$\pm$373.9& 1.4$\pm$0.9& 709.0$\pm$595.7& 2810.4$\pm$723.2& \underline{2988.3$\pm$480.2}& \textbf{3368.5$\pm$403.9}\\
  \cmidrule{2-10}
  & 2Ant-Expert&        2046.9$\pm$17.1& \textbf{2082.4$\pm$21.7}& 
  2050.0$\pm$11.9& 312.5$\pm$297.5& 2055.5$\pm$1.6& \underline{2060.0$\pm$10.3}& \underline{2060.2$\pm$20.0}& \textbf{2083.0$\pm$40.2}\\
  & 2Ant-Medium&        1422.6$\pm$21.1& 1033.9$\pm$66.4& 
  1412.4$\pm$10.9& -1710.0$\pm$1589.0& 1418.4$\pm$5.4& \underline{1428.4$\pm$14.7}& \textbf{1432.4$\pm$17.8}& \underline{1429.4$\pm$15.8}\\
  & 2Ant-MR&            995.2$\pm$52.8& 434.6$\pm$108.3&
  1016.7$\pm$53.5& -2014.2$\pm$844.7& 1105.1$\pm$88.9& 1294.5$\pm$360.2& \textbf{1498.4$\pm$20.3}& \underline{1305.5$\pm$139.1}\\
  & 2Ant-ME&            1636.1$\pm$96.0& 1800.2$\pm$21.5& 1590.2$\pm$85.6& -2992.8$\pm$7.0& 1720.3$\pm$110.6& 1740.2$\pm$158.9& \textbf{2053.3$\pm$20.4}& \underline{1974.7$\pm$116.1}\\
  \cmidrule{2-10}
  & Average&      2293.7& 2367.7& 1511.5& -611.5& 1856.4& 2430.2& 
  \underline{2749.4}& \textbf{2823.3}\\
  \bottomrule
  \end{tabular}
  }
  \caption{Comparative performance of VGM$^2$P with continuous actions environment. For the MA-MuJoCo environment, we select 3 tasks, each with 4 datasets of varying quality. For simplicity, we use \textit{ME} and \textit{MR} to represent Medium-Expert and Medium-Replay, respectively.}
  \label{result:continuous_action}
  \vskip -0.1in
\end{table*}

\section{Related Work}

\subsection{Offline Multi-Agent Reinforcement Learning}
Offline multi-agent reinforcement learning (MARL) extends offline RL from single-agent to multi-agent settings, aiming to enable effective exploration while staying within the offline data distribution and preserving coordination among agents. 
Existing methods typically build on value and policy decomposition, reducing offline MARL to independent offline RL for individual agents.
ICQ~\cite{yang2021believe} and CFCQL~\cite{shao2023counterfactual} leverage conservative Q-learning to improve exploration while maintaining coordination among agents.
OMAR~\cite{pmlr-v162-pan22a} and AlberDICE~\cite{matsunaga2023alberdice} study how multi-agent coordination affects policy improvement, while OMIGA~\cite{wang2023offline} leverages value decomposition to further enhance policy learning.
Additionally, graph-based multi-agent collaboration methods~\cite{ding2023multiagent,liu2024graph,bocheng2025graph} use mechanisms such as graph attention to build the topological structure between agents for communication.
Although these methods have made progress, the complex distributional nature of multi-agent scenarios often leads to improper credit assignment, which can hinder coordination among agents.

\subsection{Diffusion-based and Flow-based RL}
With diffusion and flow-based generative models achieving breakthroughs in image generation~\cite{ho2020denoising,lipman2023flow}, some studies begin applying them to offline RL.
Diffuser~\cite{janner2022planning} and Decision Diffusion~\cite{ajayconditional}  use diffusion models to model trajectories, while methods such as DiffusionQL~\cite{wangdiffusion} model policy. 
Despite their effectiveness, multi-step sampling in the above models significantly raises computational costs, particularly for policy learning requiring multiple iterative rollouts.
To accelerate policy learning under diffusion and flow models, EDP~\cite{kang2023efficient} uses a value-weighted diffusion training paradigm, while FQL~\cite{pmlr-v267-park25f} distills the policy into a one-step generator.
Such techniques have also attracted attention in offline MARL.
MADiff~\cite{zhu2024madiff} extends Decision Diffusion to multi-agent settings via an attention mechanism, generating trajectories that respect coordination constraints. 
DoF~\cite{li2025dof} generalizes value decomposition to distribution decomposition, naturally embedding multi-agent cooperation into diffusion-based generation. 
To improve inference efficiency, MAC-Flow~\cite{lee2025multi} and OM$^2$P~\cite{li2025om2p} extend FQL to multi-agent scenarios and use flow models to represent individual policies.
Additionally, MCGD~\cite{zenggraph} models multi-agent collaboration as a graph and enables communication using discrete and continuous diffusion models for dynamic scenarios.

\begin{figure}[t]
    \centering
    \subfigure[Ma-MuJoCo: 6HalfCheetah (continuous action)]{\includegraphics[width=1.0\linewidth]{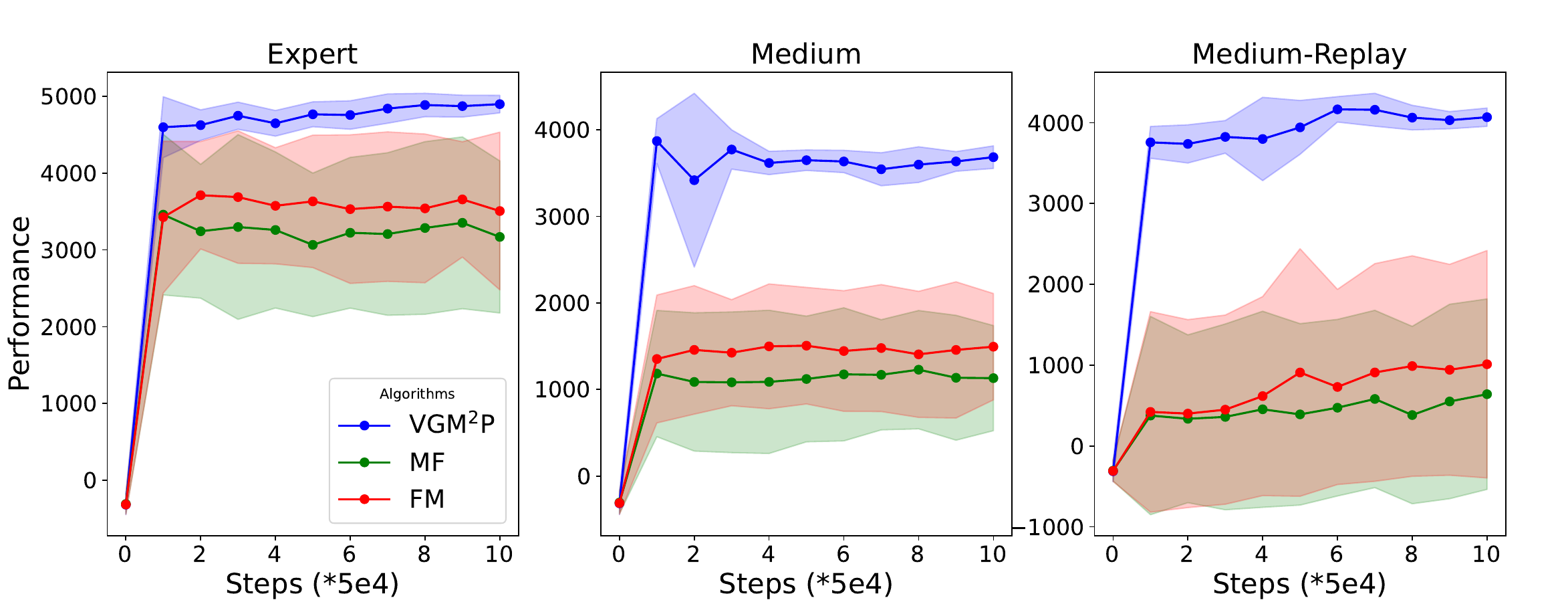}}
    \subfigure[SMACv1: 5m\_vs\_6m (discrete action)]{\includegraphics[width=1.0\linewidth]{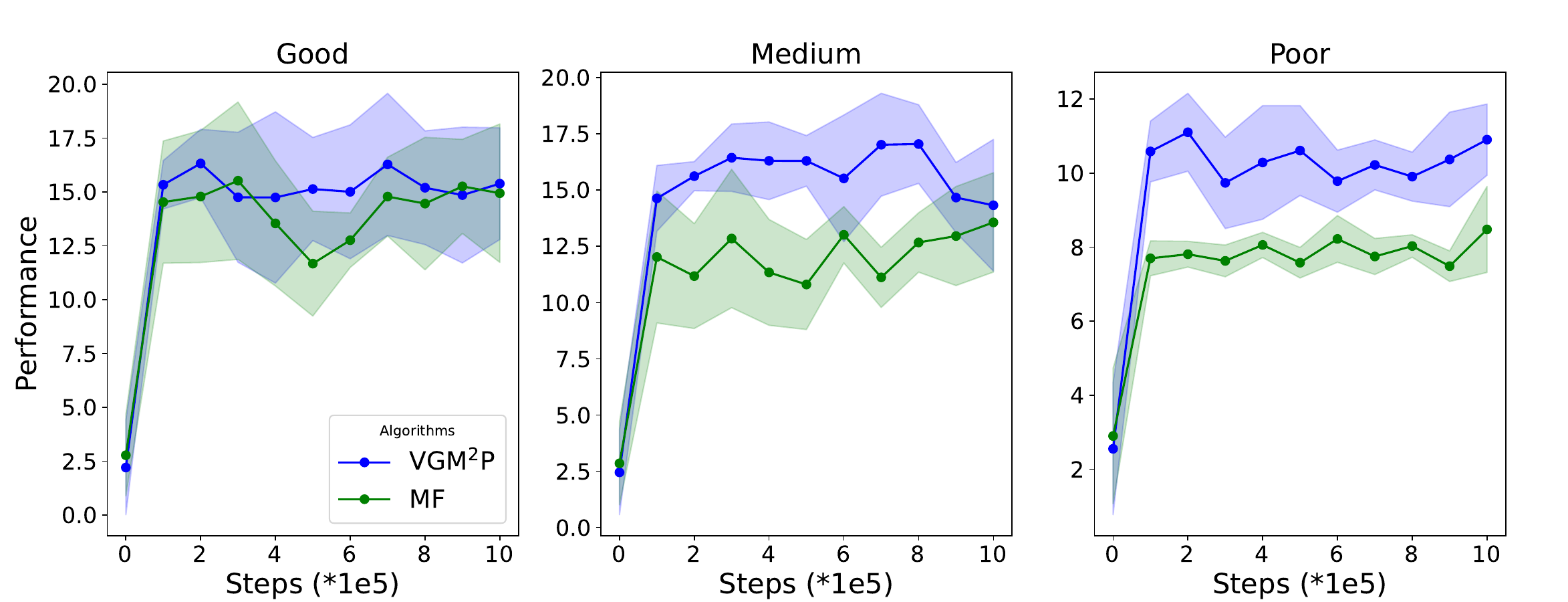}}
    \caption{The training curve between different BC and VGM$^2$P.}
    \label{fig:result_BC}
\end{figure}

\section{Experiments}

In this section, we evaluate the performance of VGM$^2$P by answering the following questions:
\begin{itemize}
    \item How does VGM$^2$P perform compared to flow-based multi-agent behavior cloning?
    \item How does VGM$^2$P perform compared to existing offline MARL methods?
    \item What factors affect the effectiveness of VGM$^2$P?
\end{itemize}

\subsection{Setup}
\noindent \textbf{Benchmarks.} 
We evaluate our method on three widely used MARL benchmarks, including two discrete action environments, StarCraft Multi-Agent Challenge (SMAC) v1 and v2~\cite{samvelyan2019starcraft}, and one continuous action one, Multi-Agent MuJoCo (MA-MuJoCo)~\cite{peng2021facmac}.

\begin{itemize}
    \item \textbf{SMAC} is a real-time combat environment with both homogeneous and heterogeneous unit settings, where agents must cooperate as a team to defeat opponents. There are two versions of datasets available~\cite{frans2025diffusion}: SMACv1 includes three quality datasets for each map, such as \textit{Good}, \textit{Medium}, and \textit{Poor}, while v2 consists of \textit{Replay} datasets with more randomized initial positions and scenarios.
    \item \textbf{MA-MuJoCo} treats the single robot as a collective of multiple agents, requiring collaboration among them to achieve a shared goal. There are four datasets of varying quality for each scenario~\cite{wang2023offline}: \textit{Expert}, \textit{Medium-Expert}, \textit{Medium-Replay}, and \textit{Medium}.
\end{itemize}

\begin{figure}[t]
    \centering
    {\includegraphics[width=0.78\linewidth]{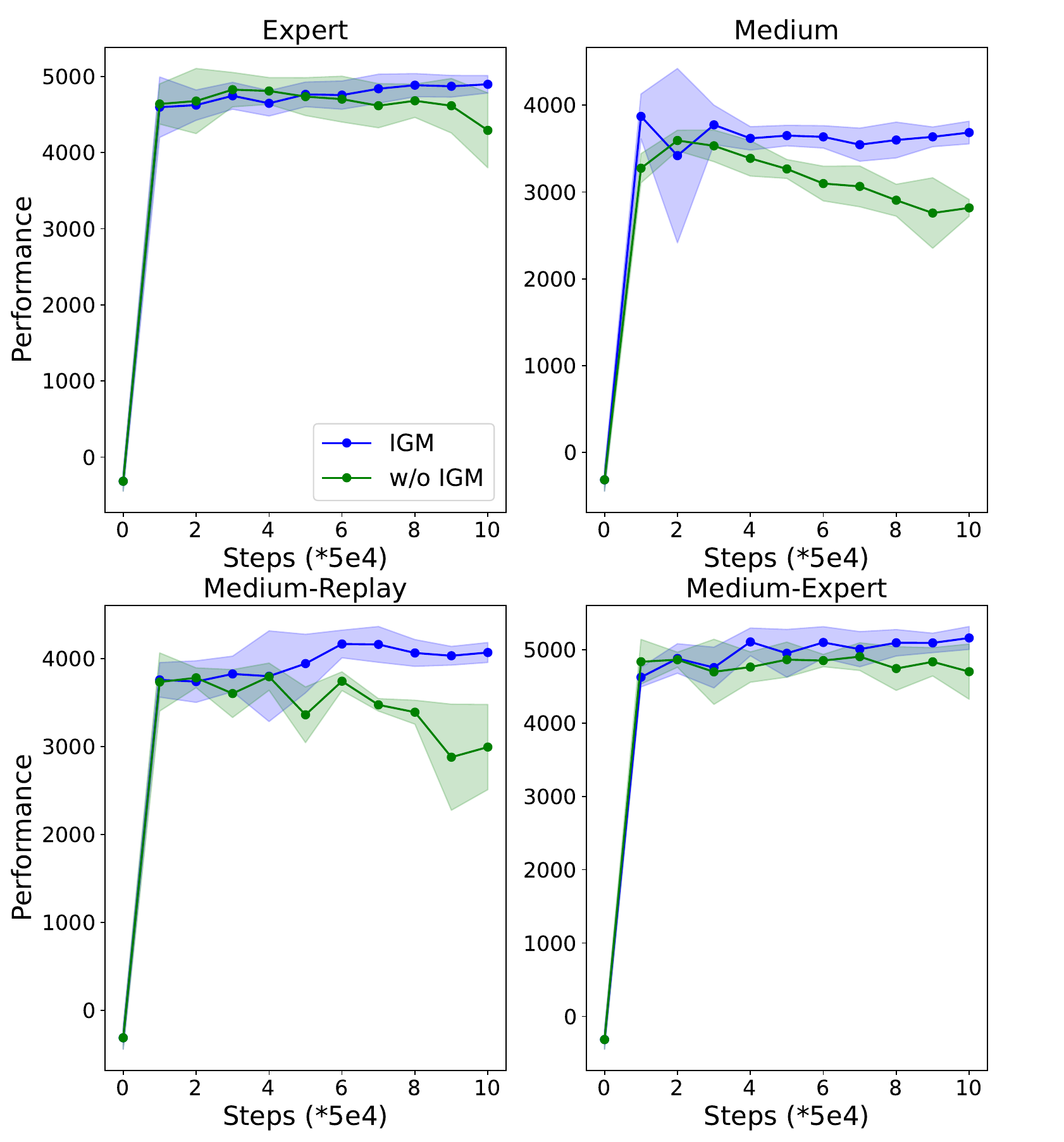}}
    \caption{The training curve for different Q-value training methods of 6HalfCheetah scenarios in MA-MuJoCo.}
    \label{fig:result_ablation_IGM}
\end{figure}

\noindent \textbf{Baselines.}
We compare $10$ representative offline MARL algorithms, covering $3$ categories: extensions of single-agent methods, recent MARL solutions, as well as diffusion- and flow-based methods.
For single-agent methods, we mainly consider BCQ~\cite{fujimoto2019off}, CQL~\cite{kumar2020conservative}, and TD3BC~\cite{fujimoto2021minimalist}. In addition, we include behavior cloning (BC) methods with different modeling paradigms (i.e., Gaussian-based, Diffusion-based, Flow Matching-based, and MeanFlow-based ones) as additional baselines.
For other methods, we consider the following:
\begin{itemize}
    \item ICQ~\cite{yang2021believe} (MARL solutions) leverages implicit conservative Q-learning for training the joint multi-agent value.
    \item OMAR~\cite{pmlr-v162-pan22a} (MARL solutions) optimizes the value function using zero-order optimization.
    \item OMIGA~\cite{wang2023offline} (MARL solutions) introduces local implicit value regularization for policy optimization.
    \item MADiff~\cite{zhu2024madiff} (Diffusion-based MARL) uses the diffusion model to model trajectories and introduces an attention mechanism. 
    \item Dof~\cite{li2025dof} (Diffusion-based MARL) decomposes the centralized diffusion model into multiple independent diffusion models.
    \item MAC-Flow~\cite{lee2025multi} (Flow-based MARL) models policy with flow matching and adopts one-step generation through distillation.
\end{itemize}
We evaluate 10 trajectories for each task and report the results based on experiments conducted with 6 seeds. We provide a detailed experimental introduction in Appendix~\ref{appendix:experiment}.

\subsection{Comparison among Behavior Cloning}

VGM$^2$P is a value-conditioned behavior cloning (BC) method that models the policy using MeanFlow. To provide a clear comparison with traditional BC, we perform unconditional BC using two generative models, Flow Matching (\textbf{FM}) and MeanFlow (\textbf{MF}), and present some comparison results in Figure~\ref{fig:result_BC}. The results show that VGM$^2$P has more advantages than traditional BC in most cases. We attribute this to the fact that, unlike traditional BC, which merely replicates behavior policy, VGM$^2$P can dig more high-reward information with value guidance conditional generation. 

\begin{figure}[t]
    \centering
    \subfigure[Ma-MuJoCo: 6HalfCheetah (continuous action)]{\includegraphics[width=1\linewidth]{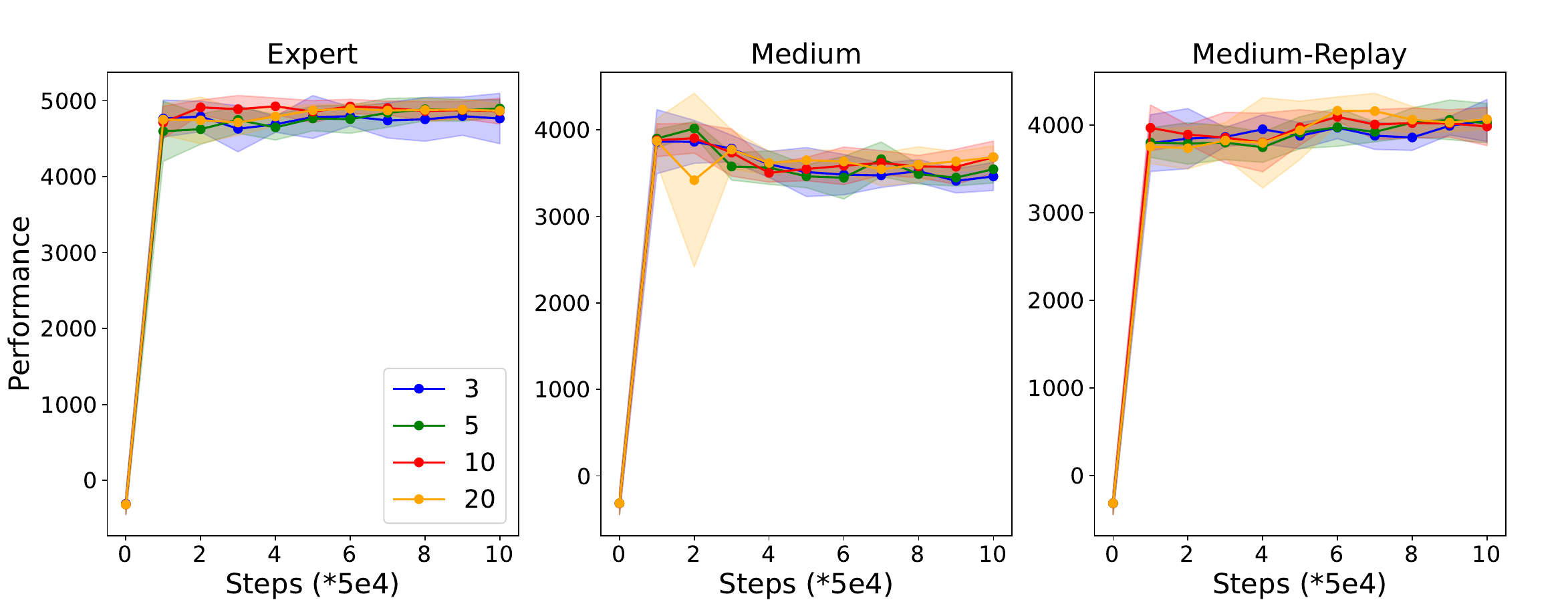}}
    \subfigure[SMACv1: 5m\_vs\_6m (discrete action)]{\includegraphics[width=1\linewidth]{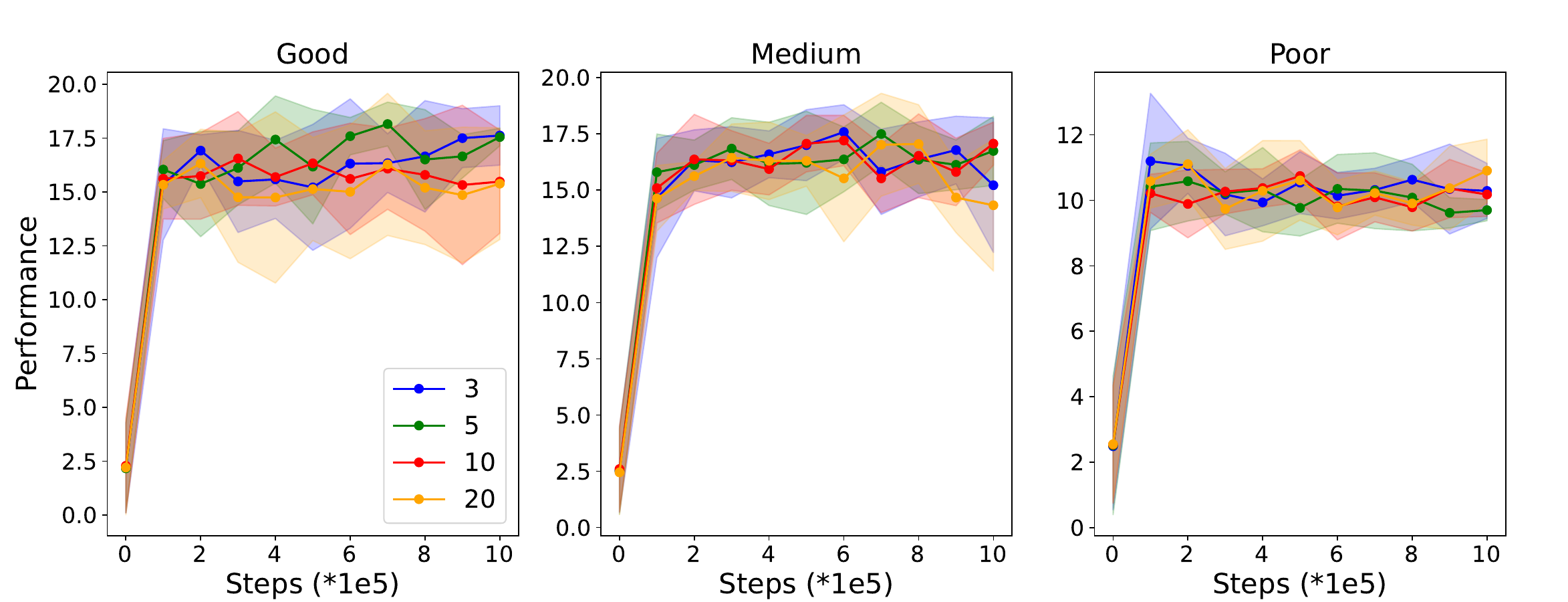}}
    \caption{The training curve for different guidance weights.}
    \label{fig:result_ablation_guidance}
\end{figure}

\subsection{Comparative Evaluation with Offline MARL}

In this experiment, we evaluate VGM$^2$P's performance in both discrete and continuous environments, comparing it with existing offline MARL methods. The results are shown in Table~\ref{result:discrete_action} and~\ref{result:continuous_action}. In simpler discrete-action multi-agent tasks, such as those in SMACv1, VGM$^2$P performs well with conditional BC; however, in SMACv2, it only outperforms traditional BC. We guess this is due to the replay dataset quality in SMACv2 not supporting VGM$^2$P's training with conditional BC. This will be a focus of our future work. To our surprise, VGM$^2$P performs comparably to existing state-of-the-art in continuous scenarios, which strongly validates the effectiveness of value guidance conditional generation. 

\begin{figure}[t]
    \centering
    \includegraphics[width=0.8\linewidth]{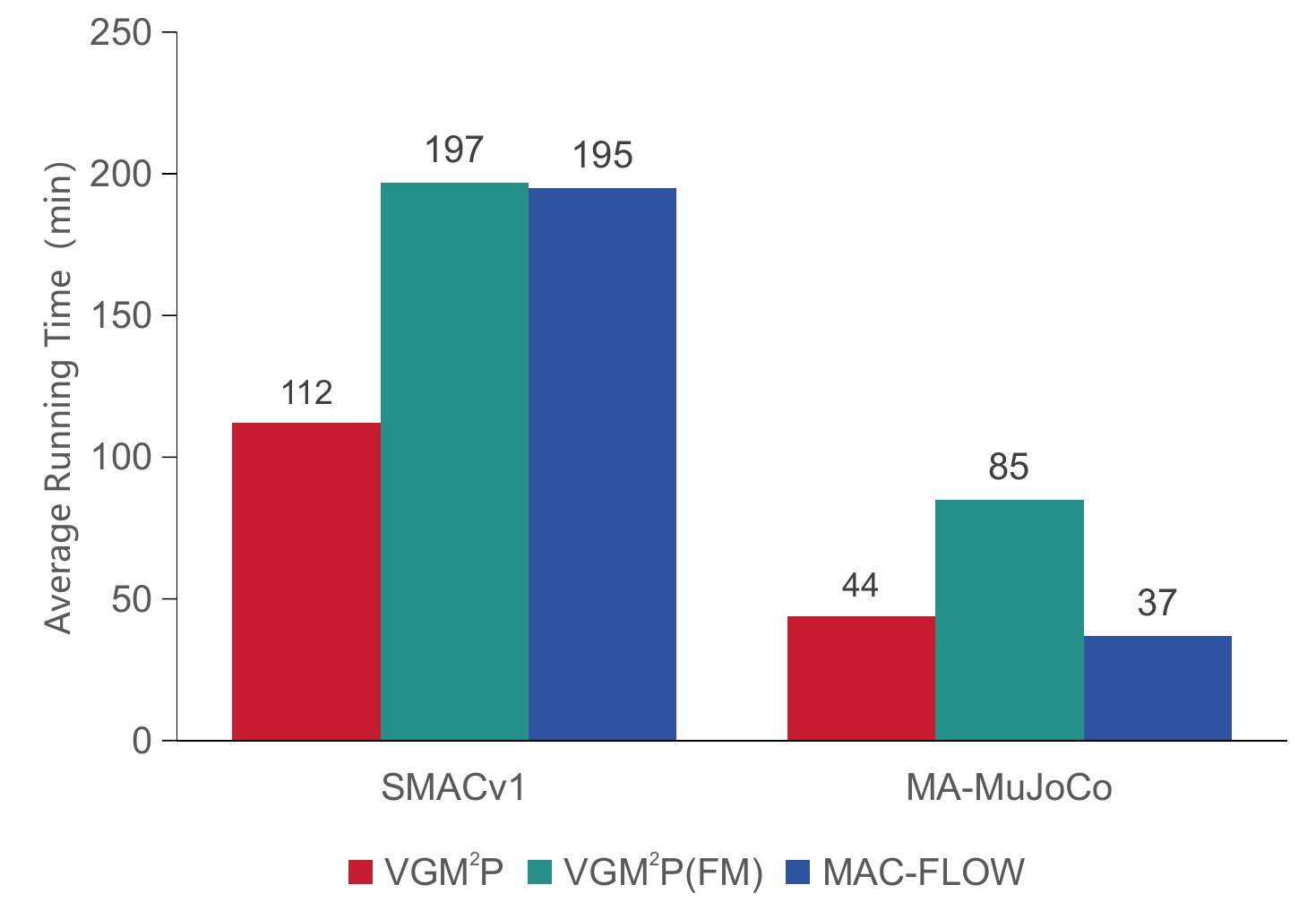}
     \caption{Comparison of running time (minutes). These results are the averages across different tasks in each environment.}
    \label{result:time}
\end{figure}

\subsection{Ablation Study}

\noindent\textbf{Effect of the Q-value training with IGM.} 
To validate the effectiveness of joint Q-value training based on the IGM principle, i.e., training with Eq.\eqref{optim:offlineMARL_value}, we compare its performance with independent training of Q-value (i.e., each agent train Q-value function with Eq.~\eqref{optim:offline_value}), as shown in Figure~\ref{fig:result_ablation_IGM}. The results show that joint training based on the IGM outperforms training independently, especially under the \textit{Medium} and \textit{Medium-Replay} datasets. 
We believe that independently Q-value function training leads multi-agent systems to converge to each local optima, neglecting global optima. In contrast, training based on the IGM principle encourages agents to explore the global optima, especially when offline data quality is low.

\noindent\textbf{Effect of the guidance coefficient.} 
To investigate the sensitivity of VGM$^2$P on the guidance coefficient, we conduct an ablation study to test its performance under different $\omega$ values.
 The result shown in Figure~\ref{fig:result_ablation_guidance} reveals that VGM$^2$P is not sensitive to guidance coefficients within a certain range, and its performance does not degrade significantly with changes in the guidance weight.

\noindent\textbf{The runtime efficiency of VGM$^2$P.} 
To evaluate the efficiency of VGM$^2$P, we compare it with the MAC-Flow, which improves efficiency through distillation and the Flow Matching version of VGM$^2$P, denoted VGM$^2$P(FM), with 10-step sampling for action generation.
The results in Figure~\ref{result:time} show that our method achieves comparable efficiency to MAC-Flow in the Ma-MuJoCo environment and is more efficient in the SMACv1 environment. Additionally, the comparison with Flow Matching highlights that VGM$^2$P's efficiency is due to MeanFlow's 1-step generation.


\section{Conclusion and Discussion}
\label{conclusion and discussion}

In this paper, we propose the value guidance multi-agent MeanFlow policy (VGM$^2$P), which leverages the advantage value 
as condition information and approximates the optimal joint policy through MeanFlow-based conditional behavior cloning. 
Experimental results show that relying solely on conditional behavior cloning, VGM$^2$P achieves performance comparable to state-of-the-art offline MARL methods.
In addition, ablation studies indicate that VGM$^2$P is both efficient and less sensitive to the guidance coefficient. 
While VGM$^2$P has yielded promising results, behavior cloning alone is insufficient for generalization in complex scenarios like SMACv2. Moreover, integrating more effective collaborative methods is expected to enhance VGM$^2$P's performance further. These will be a primary direction for our future work. 

\bibliography{reference}
\bibliographystyle{unsrt}

\clearpage

\appendix

\newtheorem{pro_app}{Proposition}

\onecolumn

\section{Experimental Details}~\label{appendix:experiment}
For the dataset, we primarily use the publicly available dataset library OG-MARL\footnote{https://huggingface.co/datasets/InstaDeepAI/og-marl}~\cite{formanek2024offthegrid}, which includes data from MARL scenarios collected through pretrained policies. Our experiments are implemented in Python with a JAX-based network architecture, and the experimental environment is Ubuntu 22.04. For computational resources, we use an RTX 3090 24GB GPU. Detailed hyperparameter settings are provided in Table~\ref{appendix:table_hyperparmeter}.

\begin{table*}[h]
\begin{center}
\small
\begin{tabular}{ll}
\toprule  
 Hyperparameter& Value\\
\hline 
 Gradient steps&                 $10^6$ (SMACv1 and SMACv2), $5\times 10^5$ (MA-MuJoCo)\\
 Batch Size&                     64\\
 Optimizer&                      Adam\\
 Learning Rate&                  $3\times10^{-4}$\\
 Model Architecture&             MLP\\
 Hidden Layer&                   4\\
 Hidden Dimension&               512\\
 Discount factor&                $0.995$ \\  
 The value of $\omega$&          $[3, 5, 10, 20]$\\
\bottomrule 
\end{tabular}
\end{center}
\caption{Hyperparameter for Meanflow model}
\label{appendix:table_hyperparmeter}
\end{table*}

\section{Proofs}~\label{appendix:proofs}
\subsection{Proof of Proposition \ref{proposition:conditional behavior policy}}~\label{appendix:proof_conditional behavior policy}
\begin{pro_app}[Value-Guidance Behavior Policy]
Given a behavior policy $\pi_{\beta}(a|o)$ and the optimal policy $\pi^*(a|o)$ derived from Eq.\eqref{optim:offlinePI}, for any variable $c\in C$ and its related distribution $p(c|o,a)$, when there exists $c^* \in C$ satisfying $p(c=c^*|o,a)\propto \exp(\frac{1}{\lambda}Q_{\pi}(o,a))$, then we have the conditional behavior policy $\pi_{\beta}(a|o,c=c^*)=\pi^*(a|o)$.
\end{pro_app}
\begin{proof}
    According to Bayes' theorem, we have
\begin{align}
    \pi_{\beta}(a|o,c) 
    = \frac{p_{\beta}(o,a,c)}{p(o,c)} 
    = \frac{p(c|o,a)\pi_{\beta}(a|o)p(o)}{p(c|o)p(o)} 
    = \frac{p(c|o,a)}{p(c|o)} \pi_{\beta}(a|o)
    =  \frac{p(c|o,a)}{\int_{a'} \pi_{\beta}(a'|o) p(c|o,a')\mathrm{d}a'} \pi_{\beta}(a|o).
\end{align}

By comparing Eq.~\eqref{optim:offlinePI}, we find that when there exists $c^* \in C$ satisfying $p(c=c^*|o,a)\propto \exp(\frac{1}{\lambda}Q_{\pi}(o,a))$ (i.e., $p(c=c^*|o,a)=k* \exp(\frac{1}{\lambda}Q_{\pi}(o,a))$, $k$ is a constant), we have: 
\begin{align}
    \pi_{\beta}(a|o,c=c^*) 
    & =  \frac{p(c=c^*|o,a)}{\int_{a'} \pi_{\beta}(a'|o) p(c=c^*|o,a')\mathrm{d}a'} \pi_{\beta}(a|o) \notag \\
    & = \frac{k * \exp(\frac{1}{\lambda}Q_{\pi}(o,a))}{\int_{a'} \pi_{\beta}(a'|o) (k* \exp(\frac{1}{\lambda}Q_{\pi}(o,a')))\mathrm{d}a'} \pi_{\beta}(a|o) \notag \\
    & = \frac{\exp(\frac{1}{\lambda}Q_{\pi}(o,a))}{\int_{a'} \pi_{\beta}(a'|o)  \exp(\frac{1}{\lambda}Q_{\pi}(o,a'))\mathrm{d}a'} \pi_{\beta}(a|o) \notag \\
    & = \pi^*(a|o).
\end{align}
\end{proof}

\subsection{Proof of Proposition \ref{proposition:policy decomposition}}~\label{appendix:proof_policy decomposition}
\begin{pro_app}
Assuming that the behavior joint policy $\pi_{\beta}^{\mathrm{tot}}(\mathbf{a}|\mathbf{o})$ and the global Q-value $Q^{\mathrm{tot}}_{\pi^{\mathrm{tot}}}(\mathbf{o},\mathbf{a})$ are decomposable, i.e., $\pi_{\beta}^{\mathrm{tot}}(\mathbf{a}|\mathbf{o}) = \prod_{i=1}^{N} \pi^{i}_{\beta}(a^i|o^i)$ and $Q^{\mathrm{tot}}_{\pi^{\mathrm{tot}}}(\mathbf{o},\mathbf{a})=\sum_{i=1}^{N} Q^i_{\phi_i} (o^i,a^i)$. 
For the optimal joint policy $\pi^{\mathrm{tot},*}(\mathbf{a}|\mathbf{o})$, when the distribution $p^i(c^i|o^i,a^i)$ satisfies $p^i(c^i=c^{i,*}|o^i,a^i)\propto \exp(\frac{1}{\lambda}Q^i_{\phi_i}(o^i,a^i))$ for each agent $i$, then we have $\prod_{i=1}^{N} \pi^{i}_{\beta}(a^i|o^i,c^i=c^{i,*})=\pi^{\mathrm{tot},*}(\mathbf{a}|\mathbf{o})$.
\end{pro_app}

\begin{proof}
When $\pi_{\beta}^{\mathrm{tot}}(\mathbf{a}|\mathbf{o}) = \prod_{i=1}^{N} \pi^{i}_{\beta}(a^i|o^i)$, $Q^{\mathrm{tot}}_{\pi^{\mathrm{tot}}}(\mathbf{o},\mathbf{a})=\sum_{i=1}^{N} Q^i_{\phi_i} (o^i,a^i)$ and $p^i(c^i=c^{i,*}|o^i,a^i)\propto \exp(\frac{1}{\lambda}Q^i_{\phi_i}(o^i,a^i))$ (i.e., $p^i(c^i=c^{i,*}|o^i,a^i)=k * \exp(\frac{1}{\lambda}Q^i_{\phi_i}(o^i,a^i))$, $k$ is a constant), we have:
\begin{align}
    & \prod_{i=0}^{N} \pi^{i}_{\beta}(a^i|o^i,c^i=c^{i,*}) \notag \\
    &= \prod_{i=0}^{N} \frac{p(c^i=c^{i,*}|o^i,a^i)}{p(c^i=c^{i,*}|o^i)} \pi^{i}_{\beta}(a^i|o^i) \notag \\
    &= \prod_{i=0}^{N}
    \frac{k*\exp(\frac{1}{\lambda}Q^i_{\phi_i}(o^i,a^i))}{\int_{\tilde{a}^i} \pi_{\beta}^i(\tilde{a}^i|o^i) (k* \exp(\frac{1}{\lambda}Q^i_{\phi_i}(o^i,\tilde{a}^i))) \mathrm{d}\tilde{a}^i} \pi^{i}_{\beta} (a^i|o^i) \notag \\
    &= \prod_{i=0}^{N}
    \frac{\exp(\frac{1}{\lambda}Q^i_{\phi_i}(o^i,a^i))}{\int_{\tilde{a}^i} \pi_{\beta}^i(\tilde{a}^i|o^i) \exp(\frac{1}{\lambda}Q^i_{\phi_i}(o^i,\tilde{a}^i)) \mathrm{d}\tilde{a}^i} \pi^{i}_{\beta} (a^i|o^i) \notag \\
    &= \prod_{i=0}^{N} \exp(\frac{1}{\lambda}Q^i_{\phi_i}(o^i,a^i)) \cdot 
    \frac{1}{ \prod_{i=0}^{N} \int_{\tilde{a}^i} \pi_{\beta}^i(\tilde{a}^i|o^i) \exp(\frac{1}{\lambda}Q^i_{\phi_i}(o^i,\tilde{a}^i))\mathrm{d}\tilde{a}^i} \cdot 
    \prod_{i=0}^{N} \pi^{i}_{\beta} (a^i|o^i) \notag \\
    &= \exp( \frac{1}{\lambda} \sum_{i=0}^{N}Q^i_{\phi_i}(o^i,a^i)) \cdot 
    \frac{1}{  \int_{\tilde{a}^1\times...\times\tilde{a}^n} \prod_{i=0}^{N}  \pi_{\beta}^i(\tilde{a}^i|o^i) \exp(\frac{1}{\lambda}Q^i_{\phi_i}(o^i,\tilde{a}^i))\mathrm{d}(\tilde{a}^1\times...\times\tilde{a}^n)} \cdot 
    \prod_{i=0}^{N} \pi^{i}_{\beta}(a^i|o^i) \notag \\
    &= \exp( \frac{1}{\lambda} \sum_{i=0}^{N} Q^i_{\phi_i}(o^i,a^i)) \cdot 
     \frac{1}{  \int_{\tilde{a}^1\times...\times\tilde{a}^n} \Bigl ( \prod_{i=0}^{N}  \pi_{\beta}^i(\tilde{a}^i|o^i) \Bigr ) \exp(\frac{1}{\lambda}\sum_{i=0}^{N}Q^i_{\phi_i}(o^i,\tilde{a}^i))  \mathrm{d}(\tilde{a}^1\times...\times\tilde{a}^n)} \cdot 
    \prod_{i=0}^{N} \pi^{i}_{\beta}(a^i|o^i) \notag \\
    &= \frac{ \exp( \frac{1}{\lambda} Q^{\mathrm{tot}}_{\pi^{\mathrm{tot}}}(\mathbf{o},\mathbf{a}))} {\int_{\tilde{\mathbf{a}}} \pi^{\mathrm{tot}}_{\beta}(\tilde{\mathbf{a}}|\mathbf{o}) \exp( \frac{1}{\lambda} Q^{\mathrm{tot}}_{\pi^{\mathrm{tot}}}(\mathbf{o},\tilde{\mathbf{a}}))  \mathrm{d} \tilde{\mathbf{a}}} \pi^{\mathrm{tot}}_{\beta}(\mathbf{a}|\mathbf{o}) \notag \\
    &= \pi^{\mathrm{tot},*}(\mathbf{a}|\mathbf{o}) 
\end{align}
\end{proof}

\newpage

\section{Learning Curves of VGM$^2$P}

\begin{figure}[h]
    \centering
    \subfigure[SMACv1:3m]{\includegraphics[width=0.5\linewidth]{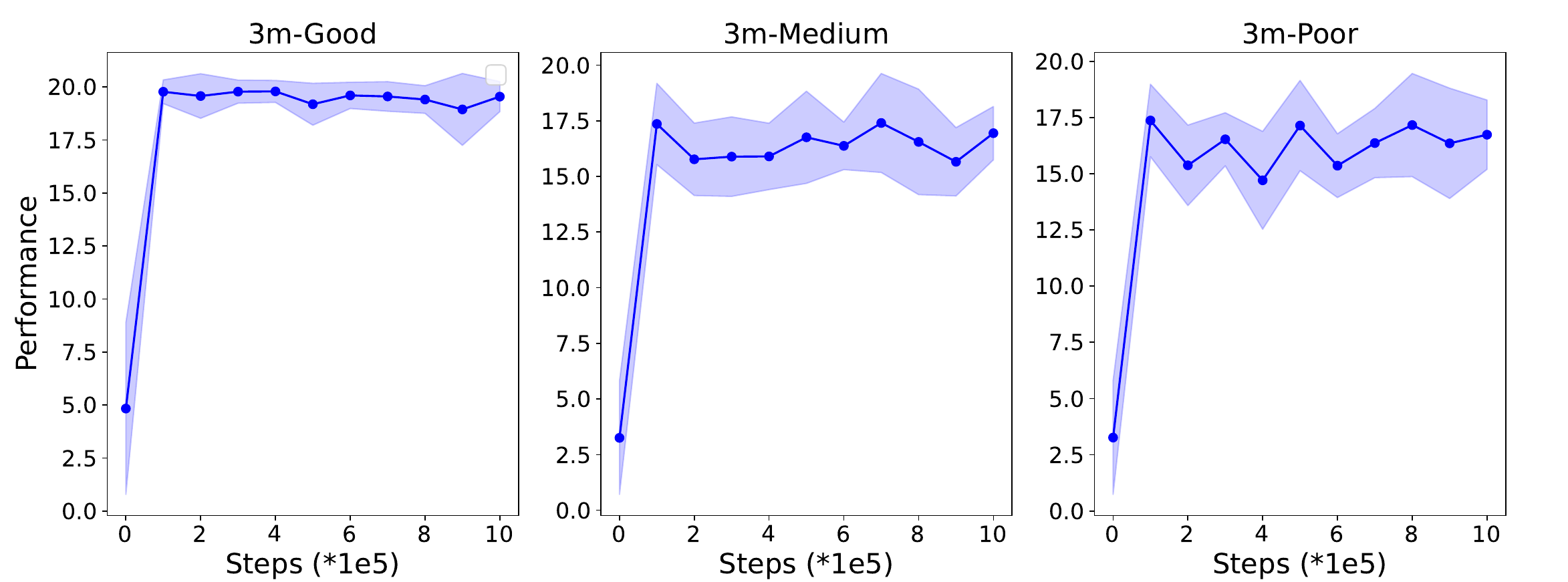}}\subfigure[SMACv1:8m]{\includegraphics[width=0.5\linewidth]{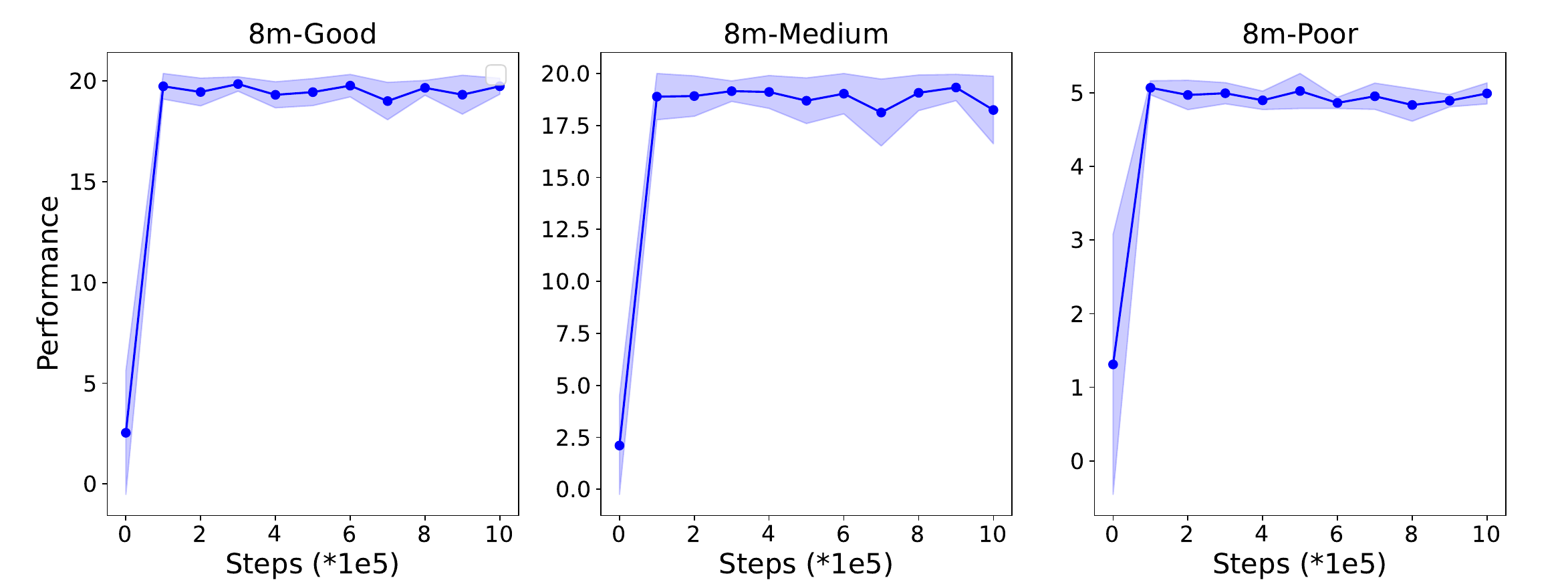}}
    \subfigure[SMACv1:2s3z]{\includegraphics[width=0.5\linewidth]{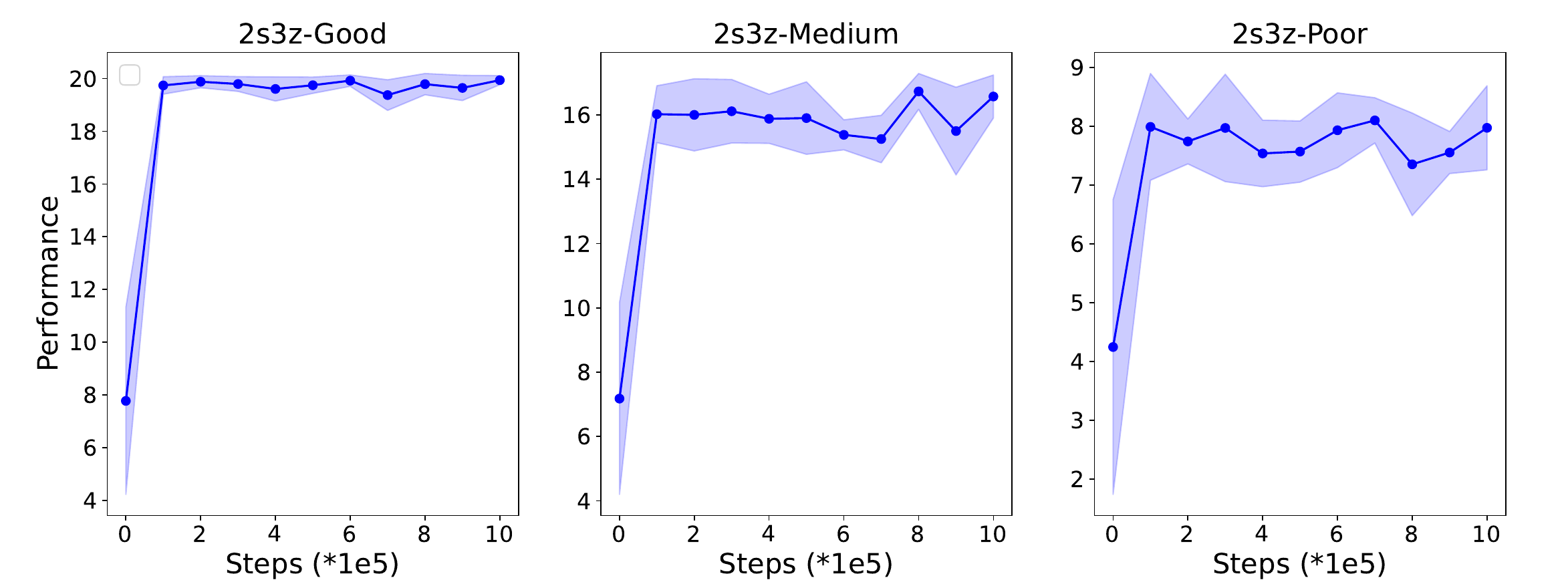}}\subfigure[SMACv1:5m\_vs\_6m]{\includegraphics[width=0.5\linewidth]{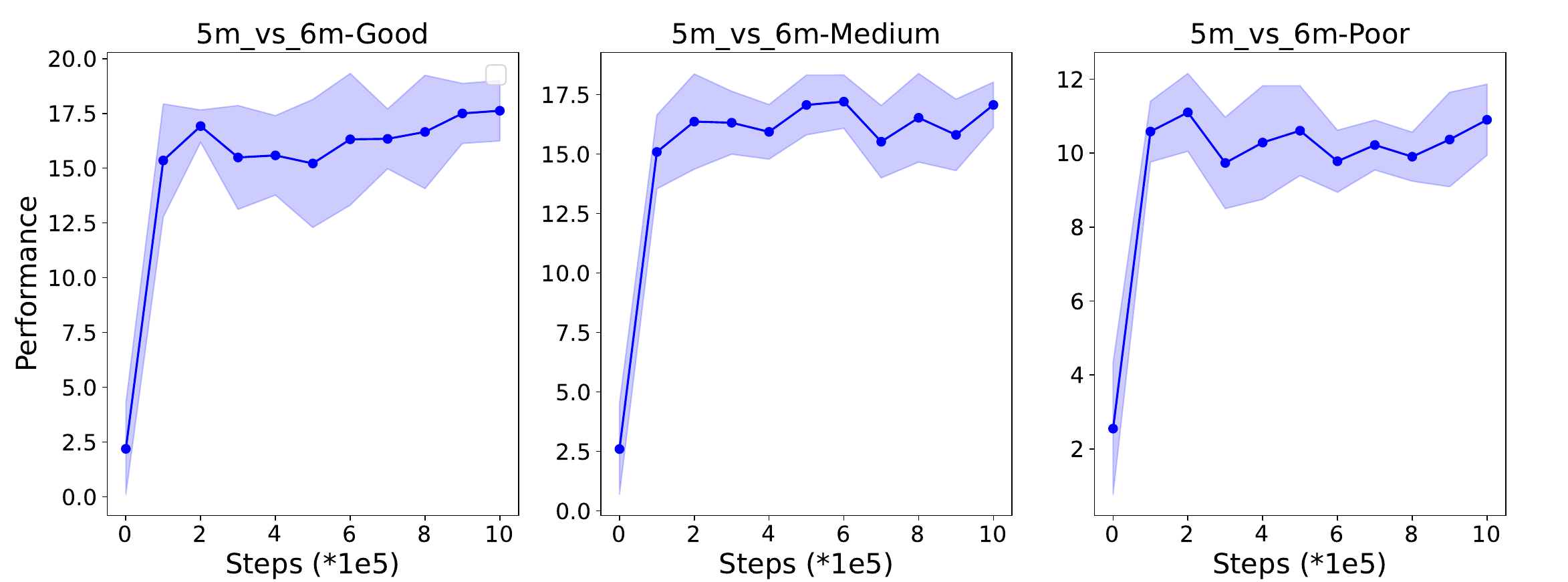}}
    \subfigure[SMACv2]{\includegraphics[width=0.7\linewidth]{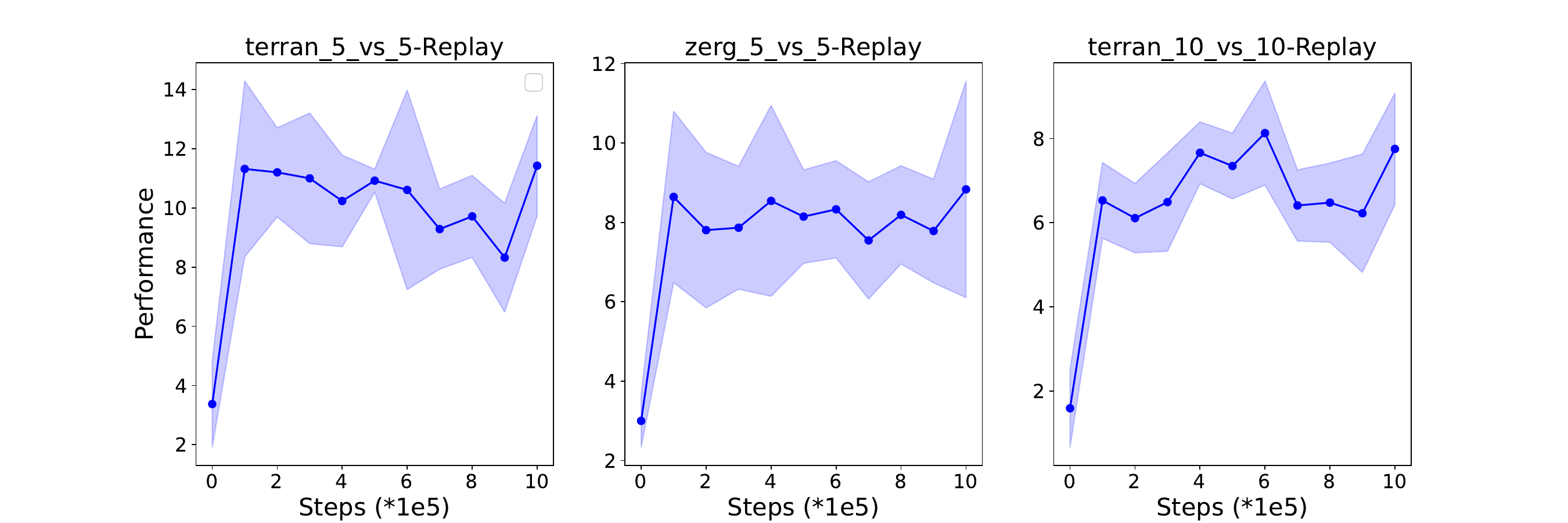}}
    \caption{The training curve for SMAC.}
    \label{fig:result_smac}
\end{figure}

\begin{figure}[h]
    \centering
    \subfigure[6HalfCheetah]{\includegraphics[width=0.8\linewidth]{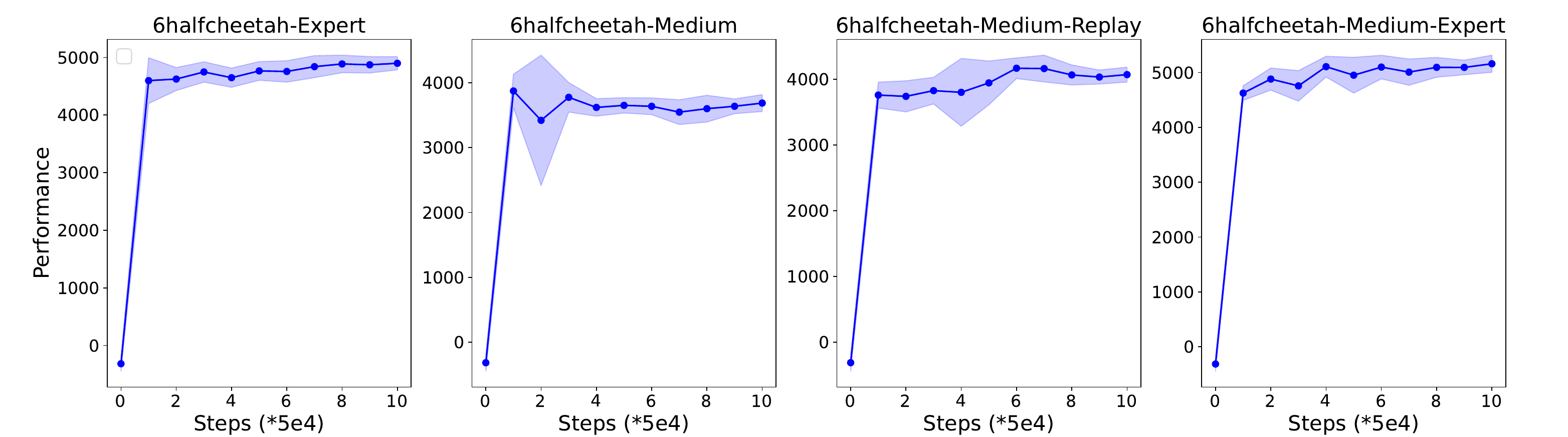}}
    \subfigure[3Hopper]{\includegraphics[width=0.8\linewidth]{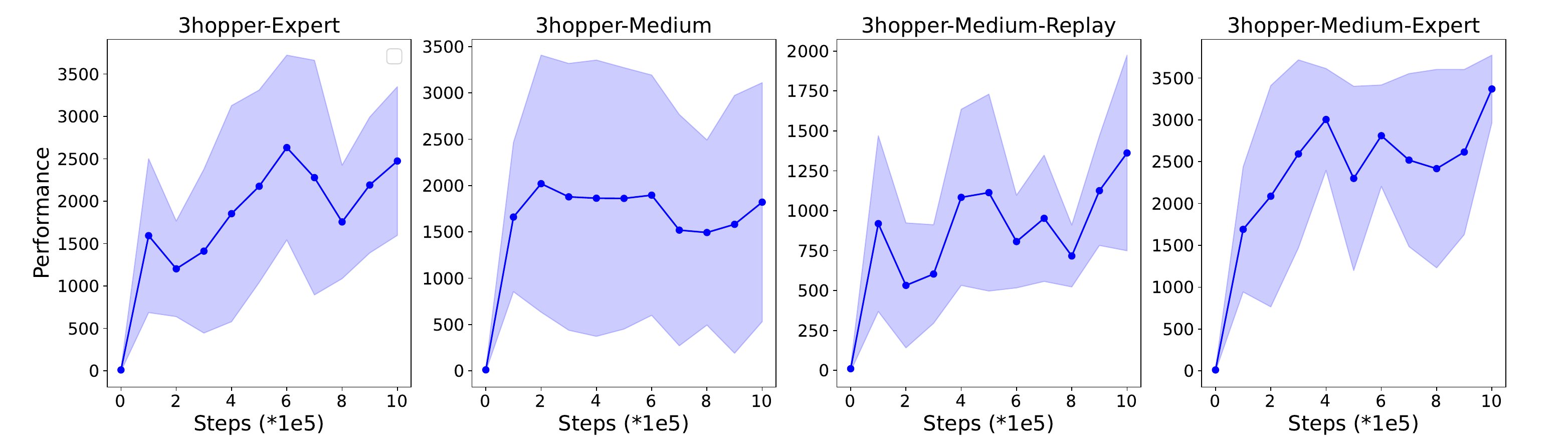}}
    \subfigure[2Ant]{\includegraphics[width=0.8\linewidth]{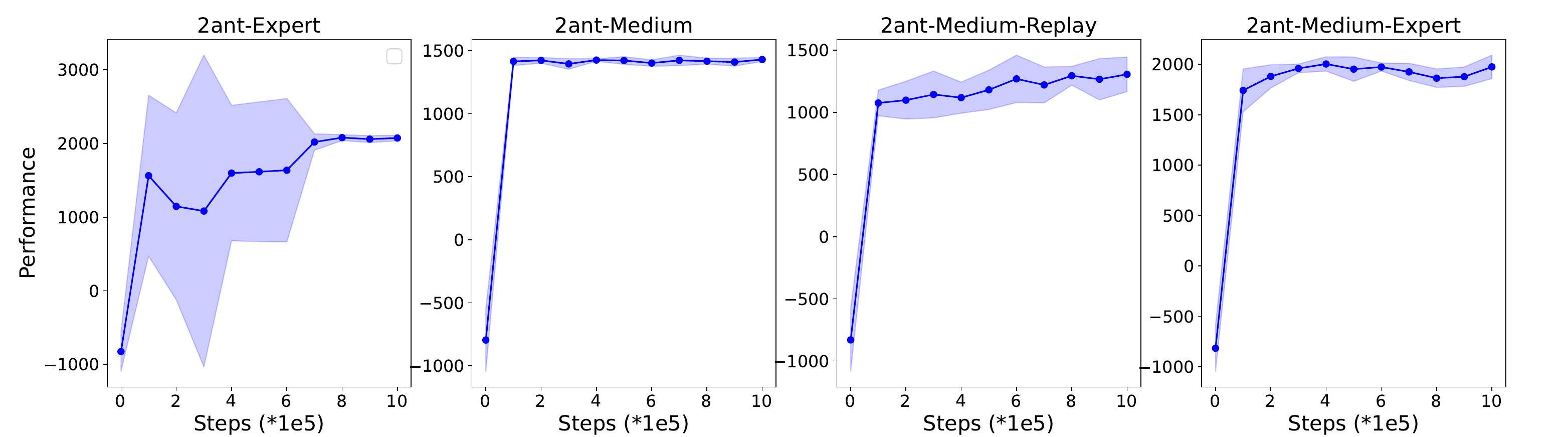}}
    \caption{The training curve for MA-MuJoCo.}
    \label{fig:result_training_mamujoco}
\end{figure}

\end{document}